\documentclass[11pt]{article} 

\usepackage{amsfonts}
\usepackage{xcolor}
\usepackage{amsthm}

\usepackage{amsmath}
\usepackage{amssymb}
\usepackage[numbers,sort&compress]{natbib}

\usepackage{multirow}
\usepackage{mathtools}
\usepackage{graphicx}
\usepackage{url}
\usepackage{algorithm2e}
\PassOptionsToPackage{algo2e,ruled}{algorithm2e}
\usepackage[margin=1in]{geometry}
\usepackage{hyperref}
\usepackage[capitalize]{cleveref}
\hypersetup{colorlinks=true,linkcolor=blue,citecolor=blue}

\usepackage[mathscr]{eucal}
\usepackage{thm-restate}

\newtheorem{theorem}{Theorem}
\newtheorem{lemma}[theorem]{Lemma}

  \newtheorem{corollary}[theorem]{Corollary}

\newtheorem{condition}{Condition}

\title{Collaborative Min-Max Regret in Grouped Multi-Armed Bandits}

\author{
    \textbf{Mo\"ise Blanchard}\\
      Columbia University\\
      \small{\texttt{mb5414@columbia.edu}}
      \and 
      \textbf{Vineet Goyal}\\
      Columbia University\\
      \small{\texttt{vgoyal@ieor.columbia.edu}}
}

\date{}

\usepackage{thmtools}
\usepackage{thm-restate}
\usepackage[mathscr]{eucal}
\usepackage{bbm}

\newcommand{\nonl}
{\renewcommand{\nl}{\let\nl\oldnl}}

\DeclareMathOperator*{\argmax}{arg\,max}
\DeclareMathOperator*{\argmin}{arg\,min}

\usepackage{latexsym,tikz,graphicx}
\usetikzlibrary{calc}
\usetikzlibrary{decorations.pathreplacing}
\usetikzlibrary{patterns}

\renewenvironment{proof}[1][]{\par\noindent{\bf Proof #1\ }}{\hfill$\blacksquare$\\[2mm]}

\begin{document}

\newcommand{\trw}{\text{\small TRW}}
\newcommand{\maxcut}{\text{\small MAXCUT}}
\newcommand{\maxcsp}{\text{\small MAXCSP}}
\newcommand{\suol}{\text{SUOL}}
\newcommand{\wuol}{\text{WUOL}}
\newcommand{\crf}{\text{CRF}}
\newcommand{\sual}{\text{SUAL}}
\newcommand{\suil}{\text{SUIL}}
\newcommand{\fs}{\text{FS}}
\newcommand{\fmv}{{\text{FMV}}}
\newcommand{\smv}{{\text{SMV}}}
\newcommand{\wsmv}{{\text{WSMV}}}
\newcommand{\trwp}{\text{\small TRW}^\prime}
\newcommand{\rhos}{\rho^\star}
\newcommand{\brhos}{\brho^\star}
\newcommand{\bzero}{{\mathbf 0}}
\newcommand{\bs}{{\mathbf s}}
\newcommand{\bw}{{\mathbf w}}
\newcommand{\bws}{\bw^\star}
\newcommand{\ws}{w^\star}
\newcommand{\Prt}{{\mathsf {Part}}}
\newcommand{\Fs}{F^\star}

\newcommand{\Hs}{{\mathsf H} }

\newcommand{\hL}{\hat{L}}
\newcommand{\hU}{\hat{U}}
\newcommand{\hu}{\hat{u}}

\newcommand{\bu}{{\mathbf u}}
\newcommand{\ubf}{{\mathbf u}}
\newcommand{\hbu}{\hat{\bu}}

\newcommand{\primal}{\textbf{Primal}}
\newcommand{\dual}{\textbf{Dual}}

\newcommand{\Ptree}{{\sf P}^{\text{tree}}}
\newcommand{\bv}{{\mathbf v}}

\newcommand{\bq}{\boldsymbol q}

\newcommand{\rvM}{\text{M}}

\newcommand{\Acal}{\mathcal{A}}
\newcommand{\Bcal}{\mathcal{B}}
\newcommand{\Ccal}{\mathcal{C}}
\newcommand{\Dcal}{\mathcal{D}}
\newcommand{\Ecal}{\mathcal{E}}
\newcommand{\Fcal}{\mathcal{F}}
\newcommand{\Gcal}{\mathcal{G}}
\newcommand{\Hcal}{\mathcal{H}}
\newcommand{\Ical}{\mathcal{I}}
\newcommand{\Jcal}{\mathcal{J}}
\newcommand{\Kcal}{\mathcal{K}}
\newcommand{\Lcal}{\mathcal{L}}
\newcommand{\Mcal}{\mathcal{M}}
\newcommand{\Ncal}{\mathcal{N}}
\newcommand{\Pcal}{\mathcal{P}}
\newcommand{\Scal}{\mathcal{S}}
\newcommand{\Tcal}{\mathcal{T}}
\newcommand{\Ucal}{\mathcal{U}}
\newcommand{\Vcal}{\mathcal{V}}
\newcommand{\Wcal}{\mathcal{W}}
\newcommand{\Xcal}{\mathcal{X}}
\newcommand{\Ycal}{\mathcal{Y}}
\newcommand{\Ocal}{\mathcal{O}}
\newcommand{\Qcal}{\mathcal{Q}}
\newcommand{\Rcal}{\mathcal{R}}

\newcommand{\brho}{\boldsymbol{\rho}}

\newcommand{\Cbb}{\mathbb{C}}
\newcommand{\Ebb}{\mathbb{E}}
\newcommand{\Nbb}{\mathbb{N}}
\newcommand{\Pbb}{\mathbb{P}}
\newcommand{\Qbb}{\mathbb{Q}}
\newcommand{\Rbb}{\mathbb{R}}
\newcommand{\Sbb}{\mathbb{S}}
\newcommand{\Vbb}{\mathbb{V}}
\newcommand{\Wbb}{\mathbb{W}}
\newcommand{\Xbb}{\mathbb{X}}
\newcommand{\Ybb}{\mathbb{Y}}
\newcommand{\Zbb}{\mathbb{Z}}

\newcommand{\Rbbp}{\Rbb_+}

\newcommand{\bX}{{\mathbf X}}
\newcommand{\bx}{{\boldsymbol x}}

\newcommand{\btheta}{\boldsymbol{\theta}}

\newcommand{\Pb}{\mathbb{P}}

\newcommand{\hPhi}{\widehat{\Phi}}

\newcommand{\Sigmah}{\widehat{\Sigma}}
\newcommand{\thetah}{\widehat{\theta}}

\newcommand{\indep}{\perp \!\!\! \perp}
\newcommand{\notindep}{\not\!\perp\!\!\!\perp}

\newcommand{\one}{\mathbbm{1}}
\newcommand{\1}{\mathbbm{1}}
\newcommand{\aprx}{\alpha}

\newcommand{\ST}{\Tcal(\Gcal)}
\newcommand{\x}{\mathsf{x}}
\newcommand{\y}{\mathsf{y}}
\newcommand{\Ybf}{\textbf{Y}}
\newcommand{\smiddle}[1]{\;\middle#1\;}

\definecolor{dark_red}{rgb}{0.2,0,0}
\newcommand{\detail}[1]{\textcolor{dark_red}{#1}}

\newcommand{\ds}[1]{{\color{red} #1}}
\newcommand{\rc}[1]{{\color{green} #1}}

\newcommand{\mb}[1]{\ensuremath{\boldsymbol{#1}}}

\newcommand{\metric}{\rho}
\newcommand{\proj}{\text{Proj}}

\newcommand{\paren}[1]{\left( #1 \right)}
\newcommand{\sqb}[1]{\left[ #1 \right]}
\newcommand{\set}[1]{\left\{ #1 \right\}}
\newcommand{\floor}[1]{\left\lfloor #1 \right\rfloor}
\newcommand{\ceil}[1]{\left\lceil #1 \right\rceil}
\newcommand{\abs}[1]{\left|#1\right|}
\newcommand{\norm}[1]{\left\|#1\right\|}

\newcommand{\todo}[1]{{\color{red} TODO: #1}}

\newcommand{\Ber}{\textnormal{Ber}}


\newcommand{\help}{H^{1}}
\newcommand{\NChelp}{H^{2}}
\newcommand{\NChelpLow}{H^{2-}}
\newcommand{\NChelpHigh}{H^{2+}}
\newcommand{\combined}{H}
\newcommand{\UCB}{\textnormal{UCB}}
\newcommand{\LCB}{\textnormal{LCB}}
\newcommand{\Hf}{\bar H}
\newcommand{\alg}{alg}
\newcommand{\Reg}{\textnormal{Reg}}
\newcommand{\MaxReg}{\textnormal{MaxReg}}
\newcommand{\Var}{\textnormal{Var}}
\newcommand{\poly}{\textnormal{poly}}
\newcommand{\TV}{\textnormal{TV}}
\newcommand{\Dkl}{\textnormal{D}_{\textnormal{KL}}}
\newcommand{\Cov}{\textnormal{Cov}}
\newcommand{\ColUCB}{\textnormal{\sf Col-UCB} }

\maketitle

\begin{abstract}
    We study the impact of sharing exploration in multi-armed bandits in a grouped setting where a set of groups have overlapping feasible action sets \cite{baek2024fair}. In this grouped bandit setting, groups share reward observations, and the objective is to minimize the \emph{collaborative regret}, defined as the maximum regret across groups. This naturally captures applications in which one aims to balance the exploration burden between groups or populations---it is known that standard algorithms can lead to significantly imbalanced exploration cost between groups.  We address this problem by introducing an algorithm \ColUCB that dynamically coordinates exploration across groups. We show that \ColUCB achieves both optimal minimax and instance-dependent collaborative regret up to logarithmic factors. These bounds are adaptive to the structure of shared action sets between groups, providing insights into when collaboration yields significant benefits over each group learning their best action independently.
\end{abstract}

\section{Introduction}
\label{sec:introduction}

\emph{Multi-armed bandits} is a classical framework for modeling decision making under uncertainty, in which users must sequentially choose an action out of uncertain alternatives \citep{bubeck2012regret,lattimore2020bandit}. 
This model has wide-ranging applications including recommender systems that sequentially offer products recommendations to users to maximize long-term revenue \citep{li2010contextual}, or clinical trials in which the goal is to find the best treatment for a sequence of patients \citep{thompson1933likelihood}.
A key aspect of multi-armed bandits is that to maximize long-term reward, one must balance between \emph{exploration} (acquiring more information about suboptimal actions to potentially improve future decisions) and \emph{exploitation} (following the best actions given the current information). 

Exploration naturally comes at a cost for users, which can disproportionately impact certain groups of users. This raises important questions about when and how the exploration burden can be alleviated for these groups \cite{raghavan2018externalities,jung2020quantifying}. Notably, \cite{baek2024fair} showed that in the asymptotic regime, the exploration cost can be shared in an arbitrarily unfair manner between groups for standard learning policies and proposed a Nash-bargaining solution to alleviate this issue. As an equivalent perspective on the problem, we can consider a setting in which several agents face their own multi-armed bandit problem, e.g., groups of recommenders targeting different populations. The goal is then to understand when and how collaborative exploration can be beneficial as opposed to each group solving their problem individually without sharing information.
We focus on heterogeneity between agents or groups in terms of their set of available actions, which corresponds to the so-called \emph{grouped bandit} setting \cite{baek2024fair}. This naturally models feasibility constraints for each group.

In the grouped bandit setting, each group $g\in\Gcal$ is associated with a set of available arms $\Acal_g\subseteq\Acal$, where $\Acal$ is the total set of arms. At each iteration, each group selects one of their available arms and receives a corresponding stochastic reward. Each group aims to minimize their reward regret compared to their best available arm in hindsight. As opposed to maximizing their reward individually, we consider the setting in which groups share information and aim to minimize the maximum regret among groups, which corresponds to the classical min-max objective \cite{rawls1971theory}. For convenience, we refer to this maximum regret between groups as the \emph{collaborative} regret. In this context, an algorithm minimizing the collaborative regret may allocate the exploration burden between groups depending on the structure of the shared arms across groups, i.e., the set system $\set{\Acal_g, g\in\Gcal}$. In particular, we aim to understand for which shared arm structures a collaborative exploration significantly improves over each group minimizing their regret individually.

\paragraph{Our Contributions.}
We introduce an algorithm \ColUCB (Collaborative-UCB) to minimize the collaborative regret, which at the high level iteratively follows the optimal allocation of pulls to explore arms estimated to be in contention for some of the groups. This is done by iteratively solving a matching problem with parameters estimated from reward observations. In particular, we show that for standard arm distributions, \ColUCB achieves the finite-sample minimax collaborative regret up to logarithmic factors (see \cref{thm:worst_case_minimax} for a formal statement):
\begin{equation*}
    \sup_{\Ical} \Ebb\sqb{\max_{g\in\Gcal} \Reg_{g,T}(\ColUCB;\Ical)} \lesssim \inf_{{\sf ALG}}\sup_{\Ical} \Ebb\sqb{\max_{g\in\Gcal} \Reg_{g,T}({\sf ALG};\Ical)}\cdot \log^2 T,
\end{equation*}
where $\Reg_{g,T}({\sf ALG};\Ical)$ denotes the regret of group $g$ after $T$ rounds following algorithm {\sf ALG} and for instance $\Ical$. We further give bounds on the minimax collaborative regret in terms of simple combinatorial properties of the shared arms. At the high level these show that whenever the average number of groups sharing each arm in any subset of arms in $\Acal$ is at least $N\gtrsim \sqrt T$, the corresponding minimax collaborative regret is of order $T^{2/3}/N^{1/3}$ (\cref{thm:quantitative_minimax}). 

Beyond these minimax guarantees, we also provide instance-dependent collaborative regret bounds for \ColUCB (\cref{thm:instance_dependent_regret}). In particular, instead of requiring all subset of arms to be sufficiently shared (as in the minimax bound), these only depend on whether the set of arms close to being optimal in some group are sufficiently shared. These instance-dependent collaborative regret bounds are also tight up to logarithmic factors (see \cref{thm:instance_lower_bound_v2}). In addition to showing that \ColUCB is essentially instance-dependent optimal, these also characterize the learning trajectory for optimal algorithms minimizing the collaborative regret.

\paragraph{Related works.}
Our work is closely related to the study of fairness in multi-armed bandits; the collaborative regret objective we consider corresponds to the classical max-min fairness criterion \cite{rawls1971theory} to optimize the regret of each group. A large body of work adresses fairness in bandits across various settings, including linear, contextual, and combinatorial bandits \cite{joseph2016fairness,jabbari2017fairness,li2019combinatorial,patil2021achieving,grazzi2022group}. Most works in this literature, however, study fairness in terms of arm or action selection, rather than optimizing the regret between different multi-armed bandit agents, which is our focus in this paper. Specifically, 
several works aim to design algorithms that guarantee similar number of pulls between actions with comparable features or expected rewards \cite{joseph2016fairness,celis2019controlling,liu2017calibrated,gillen2018online,wang2021fairness}. Another line of work directly enforces fairness constraints between arms by fixing pre-specified target pull rates or probabilities \cite{li2019combinatorial,patil2021achieving,liu2022combinatorial}.

Within the literature on fairness in bandits, our work is most closely related to \cite{baek2024fair} which introduced the grouped bandit setting and studied fair exploration in this context and in contextual bandits. They focus on the Nash bargaining solution to the problem of fairly minimizing the regret between groups and derive asymptotically optimal algorithms to do so. In comparison, in this work we focus on the min-max regret objective, and more importantly on the non-asymptotic regime, for which the geometry of the shared arms between groups plays a predominant role.

A key motivation for our work is understanding in which settings coordinated exploration can be beneficial in bandit settings where multiple groups/agents share arms. This relates our work to the broader literature on multi-agent or distributed bandits, where coordination between agents has been studied under various communication and reward-sharing models. For instance, several works study regret minimization or best-arm identification under limited communication \cite{hillel2013distributed,tao2019collaborative,wangdistributed,karpov2020collaborative}, structured or network-based communication \cite{sankararaman2019social,chawla2020gossiping,landgren2016distributed}, or coupled rewards due to collisions between arms pulled by agents \cite{liu2010distributed,rosenski2016multi,bubeck2020coordination,lugosi2022multiplayer}. The present work focuses on asymmetry between agents in terms of their feasible action sets.

\paragraph{Outline.} We formally define the setup in \cref{sec:formal_setup}, then introduce \ColUCB and state our main results in \cref{sec:main_results}. We give a brief sketch of proof for the main instance-dependent collaborative regret bound on \ColUCB in \cref{sec:proof_sketch}. We conclude in \cref{sec:conclusion}. All remaining proofs are deferred to the appendix.

\section{Formal setup and preliminaries}
\label{sec:formal_setup}

We consider the setting in which different groups in $\Gcal=\{1,\ldots,|\Gcal|\}$ interact with a set of shared arms $\Acal$ for $T\geq 2$ rounds. Each group $g\in \Gcal$ is associated with a corresponding set of feasible arms $\Acal_g\subseteq \Acal$. Without loss of generality, we assume that $\bigcup_{g\in\Gcal}\Acal_g=\Acal$ by ignoring arms that are not feasible for any group. We also assume that any group $g\in\Gcal$ has at least two actions $|\Acal_g|\geq 2$. Last, we let $c_\Gcal=1+\frac{\log \max(|\Acal|,|\Gcal|)}{\log T}$ which we will view throughout as a fixed constant. For a typical scenario, one can think of $T\geq |\Acal|,|\Gcal|$ in which case $c_\Gcal= 2$.

The learning process is as follows. At each time $t\in[T]$, each group $g\in\Gcal$ pulls a feasible arm $a_g(t)\in\Acal_g$ then incurs a reward $r_{g,t}(a)$. We suppose that the rewards are stochastic, that is, the sequences $(r_{g,t}(a))_{g\in\Gcal,t\in[T]}$ are independent identically distributed (i.i.d.) according to a distribution $\Dcal_a$, and these sequences are independent across different arms. Up to a rescaling, we suppose that all distributions $\Dcal_a$ for $a\in\Acal$ are subGaussian with parameter at most $1$, and denote by $\mu_a$ its mean. Importantly, the mean rewards are unknown to the groups; these are learned through reward observations. An instance of this grouped multi-armed bandit problem is then specified by the collection of reward distributions $\Ical=\set{\Dcal_a,a\in\Acal}$. For convenience, we introduce the notation $\mu_g^\star:=\max_{a\in\Acal_g}\mu_a$ for the maximum reward in hindsight for group $g\in\Gcal$ and let $\Delta_{g,a}:=\mu_g^\star-\mu_a$ be the suboptimality gap for action $a\in\Acal_g$.

For a given joint learning strategy {\sf ALG}, the regret incurred by group $g\in \Gcal$ on instance $\Ical$ is defined as its excess loss compared to its best arm in hindsight:
\begin{equation*}
    \Reg_{g,T}({\sf ALG};\Ical):=  \sum_{t=1}^T \mu_g^\star -\mu_{a_g(t)} = \sum_{t=1}^T \Delta_{g,a_g(t)}.
\end{equation*}
As in the standard multi-armed bandit setup, the goal is to minimize this regret for groups $\Gcal$. If we treat each group separately, the problem simply decomposes into one classical multi-armed bandit problem for each group.
Instead, we consider a collaborative setup in which the rewards observed by each group are made public at the end of each round, allowing other groups to adjust their estimates accordingly and decide which arms to pull in a collaborative manner. We measure the performance of a collaborative strategy via the regret of the worst group, which we refer to as the \emph{collaborative} regret 
\begin{equation*}
    \MaxReg_T({\sf ALG};\Ical) := \Ebb\sqb{\max_{g\in \Gcal} \Reg_{g,T}({\sf ALG};\Ical)},
\end{equation*}
where the expectation is taken with respect to the randomness of the rewards and the pulls for strategy {\sf ALG}. When deriving lower bounds, we will mainly focus on two models of reward distributions: Bernoulli distributions and unit-variance Gaussians with mean in $[0,1]$. For convenience, we denote by $\Mcal_B$ (resp. $\Mcal_G$) the set of instances $\Ical$ with Bernoulli reward distributions (resp. standard Gaussian distributions $\Ncal(\mu,\sigma^2=1)$ with mean $\mu\in[0,1]$). We expect that results can be easily generalized to other well-behaved reward distribution models. For a given model $\Mcal$ of instances, we then define the minimax collaborative regret by
\begin{equation*}
    \Rcal_{T}(\Gcal;\Mcal):=\inf_{\alg}\sup_{\Ical\in\Mcal} \MaxReg_T(\alg;\Ical),
\end{equation*}
where the infimum is taken over all joint strategies {\sf ALG} for the groups. 
Importantly, this minimax regret is dependent on the set family $\{\Acal_g,g\in\Gcal\}$, and in particular, encodes which structure of the shared arms between groups are useful to minimize the collaborative regret.

\section{Main results}
\label{sec:main_results}

\subsection{Algorithm description}

We propose an algorithm \ColUCB to minimize the collaborative regret bound which at the high-level identifies arms whose exploration should be shared among different groups, then decides of a corresponding matching between groups and actions that need to be explored.
Before describing the algorithm, we need a few notations. First, we fix a constant $C=60c_\Gcal$ which will be used within confidence bounds. As a burn-in period, the algorithm first ensures that each arm is pulled at least $n_0:=\ceil{16C\log T}$ times using the minimum number of rounds $t_{\min}$ necessary to do so (this can be implemented easily, see \cref{lemma:preliminary_t_min} for further details). 
At round $t>t_{\min}$, we denote its current number of pulls, mean estimate, suboptimality gap estimate, upper and lower confidence bound (UCB and LCB) of an arm $a\in\Acal$ as follows:\\
\begin{minipage}{0.58\textwidth}
\begin{align*}
    P_a(t) &:= \abs{\set{(s,g): s<t,\, g\in\Gcal,\, a_g(s)=a}},\\
    \hat \mu_a(t) &:=\frac{1}{P_a(t)} \sum_{s<t}\sum_{g\in\Gcal} r_s(a)\1[a_g(s)=a],\\
    \widehat\Delta_{g,a}(t) &:= \max_{a'\in\Acal_g}\hat \mu_{a'}(t)  - \hat\mu_a(t),
\end{align*}
\end{minipage}
\begin{minipage}{0.38\textwidth}
\begin{align*}
    \UCB_a(t) &:= \hat \mu_a(t) + \sqrt{\frac{C\log T}{P_a(t)}},\\
    \LCB_a(t) &:= \hat \mu_a(t) - \sqrt{\frac{C\log T}{P_a(t)}},
\end{align*}
\end{minipage}\\

\noindent We make a few remarks about the definitions of the confidence bounds. For simplicity, these use the knowledge of the horizon $T$ and hence, we assume that \ColUCB has access to this horizon parameter. Note that this assumption can be easily relieved up to modifying the confidence term to $\sqrt{C'\log (t) / P_a(t)}$ for a larger constant $C'\geq C$. Using the standard doubling trick (applying the proof to intervals of time $[2^k,2^{k+1})$), we can check that the regret bounds for this algorithm would match the regret bounds for \ColUCB provided in this paper up to an extra $\log T$ factor. Second, the constant $C$ used to define the confidence bound is larger than that for the classical UCB algorithm, which is to be expected because these will serve to estimate not only the reward of arms but also the optimal allocation of explorations. Throughout, we do not aim to optimize constant factors.

We define the set of candidate optimal actions for group $g\in\Gcal$:

\begin{equation*}
    \widehat \Acal_g(t):=\set{a\in\Acal_g: \UCB_a(t)\geq \max_{a'\in\Acal_g} \LCB_a(t)}.
\end{equation*}
These arms require further exploration if multiple arms are in contention for being the optimal arm of group $g$. Formally, the set of arms in contention for group $g\in\Gcal$, and the total set of arms still in contention at that round are defined as
\begin{equation*}
    \Ccal_g(t) := \begin{cases}
        \emptyset &\text{if } |\widehat \Acal_g(t)|=1\\
        \widehat \Acal_g(t) &\text{otherwise}.
    \end{cases} 
    \qquad \text{and} \qquad
    \Ccal(t):= \bigcup_{g\in\Gcal} \bigcap_{s\leq t}\Ccal_g(s),
\end{equation*}
where $\Ccal(0):=\Acal$.
For convenience, let us denote by $P(t):=\min_{a\in\Ccal(t)} P_a(t)$
the minimum number of pulls each arm in contention has received. To schedule the exploration of these arms in contention, we solve the following linear program:
\begin{equation}\label{eq:LP_solved_each_it}
     \Qcal(t):\qquad \begin{aligned}
q(t):=\max_{x\geq 0,q\geq 0}  q \quad
\textrm{s.t.} \quad & \sum_{a\in\Ccal(t)\cap \Acal_g} \widehat\Delta_{g,a}(t) x_{g,a} \leq \sqrt{\frac{C\log T}{P(t)}}, 
 &g\in\Gcal\\
  & \sum_{a\in\Ccal(t)\cap \Acal_g} x_{g,a} \leq 1, &g\in\Gcal\\
 & \sum_{g\in\Gcal:a\in\Acal_g} x_{g,a} \geq q, &a\in \Ccal(t),
\end{aligned} 
\end{equation}
with the convention that if $P(t)=0$ the first constraint is automatically satisfied.
Intuitively, $\Qcal(t)$ computes an allocation $x$ that ensures the maximum shared exploration for arms in contention $\Ccal(t)$, while limiting the regret incurred by each group $g\in\Gcal$ for this exploration. In particular, the term $\sum_{a\in\Ccal(t)\cap \Acal_g} \widehat\Delta_{g,a} x_{g,a} $ from the first constraint ensures that groups $g\in\Gcal$ for which arms in contention $\Ccal(t)$ are significantly worse than their estimated optimal arm participate less in the shared exploration.

We denote by $x(t) = (x_{g,a}(t))_{g\in\Gcal,a\in \Ccal(t)}$ the optimal solution to $\Qcal(t)$. Note that because of the first constraint, we have $\sum_{a\in \Ccal(t)\cap \Acal_g}x_{g,a}(t)\leq 1$ for all $g\in\Gcal$. Therefore, we can interpret $(x_{g,a})_{a\in\Ccal(t)\cap\Acal_g}$ as a partial probability distribution. Precisely, $g\in\Gcal$ pulls arm $a\in\Ccal(t)\cap \Acal_g$ with probability $x_{g,a}(t)$, and pulls the classical UCB arm with the remaining probability. The resulting algorithm \ColUCB is summarized in \cref{alg:main}. 

\begin{algorithm}[t]

\caption{\ColUCB}\label{alg:main}

\LinesNumbered
\everypar={\nl}

\hrule height\algoheightrule\kern3pt\relax
\KwIn{$\Acal_g$ for all $g\in\Gcal$, horizon $T$.}

\vspace{3mm}

\For{$t\leq t_{\min}$}{Follow any fixed schedule of pulls that ensures all arms are pulled at least $n_0$ times}

\For{$t>t_{\min}$}{
    Compute the optimal solution $x(t),q(t)$ to $\Qcal(t)$ (see \cref{eq:LP_solved_each_it}).
    For each group $g\in\Gcal$, choose
    \begin{equation*}
        a_g(t) := 
        \begin{cases}
            a &\text{w.p. } x_{g,a}(t), \quad a\in\Acal_g\cap \Ccal(t)\\
            \argmax_{a\in \Acal_g} \hat\mu_a(t) &\text{w.p. } 1-\sum_{a\in\Acal_g\cap\Ccal(t)}x_{g,a}(t)
        \end{cases}
    \end{equation*}
    
}

\hrule height\algoheightrule\kern3pt\relax
\end{algorithm}

\subsection{Worst-case collaborative regret bounds}

Our first result is that \ColUCB achieves the minimax collaborative regret up to polylog factors.

\begin{theorem}\label{thm:worst_case_minimax}
    For any horizon $T\geq 2$, the minimax collaborative regret for either Gaussian or Bernoulli bandits models $\Mcal\in\{\Mcal_G,\Mcal_B\}$ satisfies
    \begin{equation*}
        \Rcal_T(\Gcal,\Mcal) \leq \sup_{\Ical\in\Mcal} \MaxReg_T(\ColUCB;\Ical) \lesssim \Rcal_T(\Gcal,\Mcal) \cdot \log^2 T.
    \end{equation*}
\end{theorem}

The proof can be found in \cref{sec:minimax_bounds}.
While the exact minimax collaborative regret $\Rcal_T(\Gcal)$ can depend in complex ways on $\{\Acal_g,g\in\Gcal\}$, upper and lower bounds on this minimax bound can be written in terms of simple combinatorial properties of the arms set system.
As can be expected, the collaborative regret depends on the amount of shared arms between groups. To give an example, consider the extreme case in which all groups are disjoint: $\Acal_g\cap\Acal_{g'}=\emptyset$ for all groups $g\neq g'$. Then, sharing information about rewards between groups is irrelevant and the problem again decomposes into one multi-armed bandit instance for each group. The corresponding worst-case collaborative regret is then naturally $\Ocal(\sqrt{\max_{g\in \Gcal} |\Acal_g|\cdot  T})$. On other extreme, if all arms are shared $\Acal_g=\Acal$ for all $g\in \Gcal$, this essentially corresponds to a single multi-armed bandit instance with longer horizon $|\Gcal|T$. The worst-case collaborative regret then becomes $\Ocal(\sqrt{ |\Acal|T/|\Gcal|})$ since the total regret can be evenly distributed between groups.

More generally, the collaborative regret will be reduced when the average number of groups sharing each arm is large. This is quantified by the average number of groups that can help the exploration of arms in any subset $S\subseteq\Acal$:
\begin{equation*}
    \help(S) := \frac{\abs{\set{g\in\Gcal:\Acal_g\cap S\neq\emptyset}}}{|S|}.
\end{equation*}
Among these groups, those which have near-optimal arms in $S$ are particularly important since they can contribute to the shared exploration at a much lower regret cost. We then introduce two quantities to serve as lower and upper bounds on the number of such groups. For convenience, for any set of groups $\Gcal'\subseteq \Gcal$ we denote by $\Cov(\Gcal'):=\bigcup_{g'\in\Gcal'}\Acal_{g'}$ the set of arms that groups in $\Gcal'$ cover. For $S\subseteq \Acal$, we define
\begin{align*}
    \NChelpLow(S)&:= \frac{1}{|S|}\cdot \min\set{|\Gcal'|:  \Gcal'\subseteq \Gcal ,S\subseteq \Cov(\Gcal')} \\
    \NChelpHigh(S)&:= \frac{1}{|S|}\cdot \min_{\substack{\Gcal'\subseteq \Gcal \text{ s.t.} \\ S\subseteq \Cov(\Gcal')}} \abs{\set{g\in\Gcal:\Acal_g\cap S\neq\emptyset,\; \Acal_g\subseteq \Cov(\Gcal') } } .
\end{align*}

In words, $\NChelpLow(S)$ counts the average minimum number of groups needed to cover the set $S$, while $\NChelpHigh$ counts the minimum (average) total number of groups that are covered by a cover of $S$. The set of groups $\Gcal'$ covering the arms $S$ will effectively correspond to groups $g\in\Gcal$ for which arms $\Acal_g\cap S$ are in contention for being optimal for $g$ and require additional exploration. In a pessimistic scenario, only groups $\Gcal'$ incur a low cost to explore these arms $S$ as quantified by $\NChelpLow(S)$. However, in a more optimistic scenario, groups that are covered by $\Cov(\Gcal')$ can potentially incur low cost also if arms in $S$ are approximately optimal within $\Cov(\Gcal')$. This is quantified by $\NChelpHigh(S)$.

Next, we combine the two types of exploration help from $\help(S)$ and $\NChelpLow(S),\NChelpHigh(S)$. We treat the terms from $\NChelpLow(S),\NChelpHigh(S)$ as correcting terms compared to the average number of groups sharing each arm $\help(S)$ and define
\begin{equation*}
    \combined_T^-(S):= \help(S) + \NChelpLow(S)^{3/2}\sqrt T \quad \text{and} \quad
    \combined_T^+(S):= \help(S) + \NChelpHigh(S)^{3/2}\sqrt T.
\end{equation*}
This quantifies the total exploration help groups can provide for arms in $S$ when balancing the two forms of exploration help. To be more precise, when needed to explore arms in $S$, the bottleneck corresponds to the worst-case subset of arms $S'\subseteq S$. This is captured in the following final quantity
\begin{equation}\label{eq:def_final_help}
    \Hf_T^+(S) :=\min_{\emptyset\subsetneq S'\subseteq S} \combined_T^+(S') \quad \text{and} \quad \Hf_T^-(S) :=\min_{\emptyset\subsetneq S'\subseteq S} \combined_T^-(S').
\end{equation} 
The following result gives quantitative bounds on the minimax collaborative regret using $\Hf_T^+(\Acal),\Hf_T^-(\Acal)$. Note that in many cases, both quantities can coincide. For instance, consider the case where $\Gcal$ contains all groups with $k$ arms within $\Acal$, where $k\leq|\Acal|$. Then, $\combined_T^+(S)$ and $\combined_T^-(S)$ are both minimized for sets $S\subseteq \Acal$ with $|S|=k$, which gives $\Hf_T^+(\Acal)=\Hf_T^-(\Acal)=(\sqrt{T/k}+1)/k$.

\begin{theorem}\label{thm:quantitative_minimax}
    For any horizon $T\geq 2$, the minimax collaborative regret for either Gaussian or Bernoulli reward model $\Mcal\in\{\Mcal_G,\Mcal_B\}$ satisfies
    \begin{equation*}
        \min\paren{\frac{T^{2/3}}{\Hf_T^+(\Acal)^{1/3}}, T} 
        \lesssim \Rcal_T(\Gcal,\Mcal) \leq \sup_{\Ical\in\Mcal} \MaxReg_T(\ColUCB;\Ical) \lesssim \frac{T^{2/3}}{\Hf_T^-(\Acal)^{1/3}} \log T
    \end{equation*}
\end{theorem}

We present the proof in \cref{sec:minimax_bounds}.
As a simple remark, \cref{thm:quantitative_minimax} implies that the collaborative regret is at most the regret where each group minimizes their regret individually. Indeed, for any set $\emptyset\subsetneq S\subseteq\Acal$, we have $\NChelpLow(S) \geq 1/\max_{g\in\Gcal}|\Acal_g|$ since any group can cover at most $\max_{g\in\Gcal}|\Acal_g|$ arms in $S$. Hence, the upper bound for the collaborative regret of \ColUCB is at most $\widetilde\Ocal(T^{2/3} /\Hf_T^-(\Acal)) \leq \widetilde\Ocal(\sqrt{\max_{g\in\Gcal} |\Acal_g|\cdot T})$. Also, note that \cref{thm:quantitative_minimax} recovers this worst-case collaborative regret in the individual setting up to log factors since $\NChelpLow=\NChelpHigh$ when the arm sets $\{\Acal_g,g\in\Gcal\}$ are all disjoint.

More importantly, \cref{thm:quantitative_minimax} gives insights into when sharing exploration between groups is helpful compared to minimizing regret individually. As an example, whenever
\begin{equation*}
    \min_{\emptyset\subsetneq S\subsetneq\Acal} \help(S) \gg \frac{\sqrt T}{\max_{g\in\Gcal}|\Acal_g|^{3/2}},
\end{equation*}
the worst-case collaborative regret is strictly improved by sharing the exploration effort compared to the groups behaving individually (see \cref{lemma:sufficient_condition_improve} for a proof).
On the other hand, if we omit the help from $\NChelpHigh,\NChelpLow$, considering the lower bound of \cref{thm:quantitative_minimax} we do not necessarily always expect a significant improvement of the worst-case collaborative regret compared to individual regret minimization when $\min_{\emptyset\subsetneq S\subsetneq\Acal} \help(S) \ll \sqrt T$.

\subsection{Instance-dependent collaborative regret bounds}

In this section, we turn to instance-dependent collaborative regret bounds for \ColUCB. In many cases these will significantly improve over the bounds from \cref{thm:quantitative_minimax} which can be overly conservative in specific instances.
The worst-case regret bounds depend on $\Hf_T^+(\Acal)$ and $\Hf_T^-(\Acal)$. By their definition in \cref{eq:def_final_help}, these focus on the exploration help for the worst-case subset of arms $\emptyset\subsetneq S\subseteq \Acal$. For a specific instance $\Ical$, however, it suffices to focus the shared exploration effort on undifferentiated arms that are in contention for being optimal in one of the groups. Precisely, we first define the minimum suboptimality gap for group $g\in\Gcal$ as $\Delta_{g,\min}:=\min\{\mu_{a_1}-\mu_{a_2},\,a_1\neq a_2\in\Acal_g,\, \mu_{a_1}=\mu_g^\star\}$. Then, for any tolerance $\epsilon>0$, we can define the set of arms in $\epsilon$-contention as
\begin{equation*}\Ccal^\star(\epsilon;\Ical):=\bigcup_{g\in\Gcal:\Delta_{g,\min}\leq \epsilon}\set{a\in\Acal_g:
    \mu_g^\star - \mu_a \leq \epsilon }.
\end{equation*}
Intuitively, an algorithm minimizing the collaborative regret should focus its shared exploration on arms in $\Ccal^\star(\epsilon;\Ical)$ for an adequate tolerance parameter $\epsilon>0$.

Next, we define the following linear program for any tolerance $\epsilon\geq 0$,
\begin{equation*}
    \begin{aligned}
M(\epsilon;\Ical) := \max_{x\geq 0,M\geq 0}    M \quad \textrm{s.t.} \quad & \sum_{a\in  \Acal_g} \max\paren{ \Delta_{g,a},\epsilon} \cdot x_{g,a} \leq 1, 
 &g\in\Gcal\\
 & \sum_{g\in\Gcal:a\in\Acal_g} x_{g,a} \geq M, &a\in \Ccal^\star(\epsilon;\Ical).
\end{aligned} 
\end{equation*}
Solving this linear program essentially determines the optimal shared exploration for arms in $\Ccal^\star(\epsilon;\Ical)$.
Note that since the sets $\Ccal^\star(\epsilon;\Ical)$ are non-increasing as the tolerance parameter decays $\epsilon\to0^+$, $M(\cdot;\Ical)$ is non-decreasing as $\epsilon\to 0^+$. 
With this quantity at hand, we define the two following functions:
\begin{equation*}
    T(\epsilon;\Ical,\sigma):=\int_{\epsilon}^{1} \frac{\sigma\cdot dz}{M(\sigma z;\Ical)z^4} \quad\text{and}\quad  R(\epsilon;\Ical,\sigma):=\int_{\epsilon}^{1} \frac{\sigma^2\cdot dz}{M(\sigma z;\Ical)z^3},
\end{equation*}
for $\sigma>0$ and $\epsilon\in(0,\sigma]$. Here, $\sigma$ will correspond to the noise parameter for the reward distributions. By default, we will write $T(\epsilon;\Ical):=T(\epsilon;\Ical,1)$ and $R(\epsilon;\Ical):=R(\epsilon;\Ical,1)$.

At the high level, $T(\epsilon;\Ical)$ estimates the horizon necessary to determine optimal actions up to the tolerance $\epsilon$ for all groups, while $R(\epsilon;\Ical)$ measures the collaborative regret accumulated up until that horizon. More precisely, these would correspond to the horizon and regret necessary in hindsight, to certify which are arms in $\epsilon$-contention in $\Ical$ from the observed reward samples. The following result gives an instance-dependent collaborative regret bounds on \ColUCB in terms of the functions $T(\cdot;\Ical)$ and $R(\cdot;\Ical)$. We give a proof sketch in \cref{sec:proof_sketch} and give full proof details in \cref{sec:instance_dependent_UB}.

\begin{theorem}\label{thm:instance_dependent_regret}
    There exists a universal constant $c_0>0$ such that the following holds. Fix any instance $\Ical$ with value distributions that are subGaussian with parameter $\sigma^2$ for $\sigma>0$. Denote $\Delta_{\max} := \max_{g\in\Gcal,a\in\Acal_g}\Delta_{g,a}$. Then, for any $T\geq 2$, 
    \begin{equation*}
         \MaxReg_T(\ColUCB;\Ical) \leq c_0 c_\Gcal\cdot  ( R(\epsilon;\Ical,\sigma)\log T + t_{\min}\Delta_{\max}),
    \end{equation*}
    where $\epsilon\in(0,1]$ satisfies $T(\epsilon;\Ical,\sigma)= T$.
\end{theorem}

Note that any Bernoulli distribution is subGaussian with parameter $1$: specialized to this case, we can replace $T(\epsilon;\Ical,\sigma)$ (resp. $R(\epsilon;\Ical,\sigma)$) with $T(\epsilon;\Ical)$ (resp. $R(\epsilon;\sigma)$) within the statement of \cref{thm:instance_lower_bound_v2}.
We can check that the definition of $\epsilon\in(0,1]$ such that $T(\epsilon;\Ical,\sigma)=T$ is also well-defined since $T(\cdot;\Ical,\sigma)$ forms a bijection from $(0,1]$ to $\Rbb_+$ (see \cref{lemma:continuity_R_T}). The second term $t_{\max}\Delta_{\max}$ from \cref{thm:instance_lower_bound_v2} is necessary in some sense: $t_{\min}/\log T$ rounds are necessary to ensure that each arm is pulled at least once (see \cref{lemma:preliminary_t_min}).

The regret bound from \cref{thm:instance_dependent_regret} implies that \ColUCB adapts to each instance since the functionals $T(\cdot;\Ical)$ and $R(\cdot;\Ical)$ are defined with respect to sets of arms in contention for the specific instance $\Ical$ (we recall that \ColUCB does not have access to the sets $\Ccal^\star(\epsilon;\Ical)$). Following the discussion above, $R(\epsilon;T)$ corresponds to the optimal regret for finding optimal actions up to $\epsilon$ for all groups. This is formalized in the following instance-dependent lower bound.

\begin{theorem}\label{thm:instance_lower_bound_v2}
    There exists a universal constants $c_0,c_1>0$ such that the following holds. Let {\sf ALG} be an algorithm for the groups and $T\geq 2$. Suppose that for some Gaussian instance $\Ical$ with reward distribution variance $1$, its collaborative regret satisfies
    \begin{equation*}
        \MaxReg_T({\sf ALG};\Ical)\leq c_0 \frac{R(\epsilon;\Ical)}{\log T},
    \end{equation*}
    where $\epsilon\in(0,1]$ satisfies $T(\epsilon;\Ical)= T$. Then, there exists another Gaussian instance $\Ical'$ obtained by modifying the mean value of a single arm, for which
    \begin{equation*}
        \MaxReg_T({\sf ALG};\Ical')\geq c_1 \frac{R(\epsilon;\Ical)}{\log T}.
    \end{equation*}
\end{theorem}

The proof can be found in \cref{sec:instance_dependent_LB}.
Importantly, the lower bounds from \cref{thm:instance_lower_bound_v2} imply that \ColUCB is instance-dependent optimal up to logarithmic factors.

As a last remark, the definitions $T(\cdot;\Ical)$ and $R(\cdot;\Ical)$ in their integral form are somewhat involved to capture instances in which $M(\epsilon;\Ical)$ may grow wildly as $\epsilon\to 0^+$. However, in typical scenarios where this growth is controlled, the bound from \cref{thm:instance_dependent_regret} can be significantly simplified. For example, if the instance satisfies the following condition, we present a simplified regret bound below.

\begin{condition}\label{condition:nice_instance}
    There exist constants $C_1\geq 1$ and $\alpha\in(0,2]$ such that for any $0<z_1\leq z_2\leq 1$, 
    \begin{equation*}
        M(z_1;\Ical) \leq C_1 \paren{\frac{z_2}{z_1}}^{2-\alpha} M(z_2;\Ical).
    \end{equation*}
\end{condition}

Under \cref{condition:nice_instance} we can rewrite \cref{thm:instance_dependent_regret} as follows. We include a proof in \cref{subsec:simplified_statement}.

\begin{corollary}\label{cor:simplified_instance_dependent_regret}
    Fix an instance $\Ical$ satisfying \cref{condition:nice_instance} for the constants $C_1\geq 1$ and $\alpha\in(0,2]$. Then, for any $T\geq 1$,
    \begin{equation*}
        \MaxReg_T(\ColUCB;\Ical) \lesssim \epsilon^\star T\log T +t_{\min}\Delta_{\max},
    \end{equation*}
    where $\epsilon^\star = \min\set{z\geq (0,1]: M(z;\Ical)\cdot z^3 T \geq 1}\cup\{1\}$, and where $\lesssim$ only hide constant factors depending on $c_\Gcal$, $C_1$, and $\alpha$.
\end{corollary}

To provide some intuitions about this result: consider groups that have arms in $\epsilon$-contention. These incur at least a per-round regret $\epsilon$ until each arm in $\epsilon$-contention has been disambiguated. This requires $\Ocal(1/\epsilon^2)$ pulls for each arm within $\Ccal^\star(\epsilon;\Ical)$. To maximize exploration of these arms, other groups can share exploration up to the maximum per-round regret $\epsilon$ (otherwise the collaborative regret is dominated by the shared exploration). By design, $ M(\epsilon;\Ical)\cdot \epsilon$ measures the maximum average per-round number of pulls for arms in $\epsilon$-contention. As a result, the threshold $\epsilon^\star$ from \cref{cor:simplified_instance_dependent_regret} intuitively corresponds to the minimum error tolerance $\epsilon^\star$ for which disambiguating arms in $\epsilon^\star$-contention is achievable within $T$ rounds, even for an algorithm with prior knowledge of the optimal threshold $\epsilon^\star$ and $\Ccal(\epsilon^\star;\Ical)$. The regret bound then implies that for instances satisfying the technical \cref{condition:nice_instance}, \ColUCB essentially only suffers this minimum regret.

\section{Proof sketch for \cref{thm:instance_dependent_regret}}
\label{sec:proof_sketch}

In this section, we give the main steps to proving the instance-dependent regret bound for \cref{alg:main} from \cref{thm:instance_dependent_regret}. Throughout, we fix the horizon $T\geq 1$ and an instance $\Ical$ with value distributions that are subGaussian with parameter at most $1$.
Recall that at each iteration $t\in[T]$, \cref{alg:main} uses the solution to the linear program $\Qcal(t)$ to decide of its allocation strategy. The first step of the proof is to relate this to the linear program defining $M(\epsilon;\Ical)$. We define appropriate value for $\epsilon$ as follows:
\begin{equation*}
    \epsilon(t):=\sqrt{\frac{C\log T}{P(t)}},
\end{equation*}
which corresponds exactly to the tolerance term used in one of the constraints from $\Qcal(t)$.
To do so, we need to check that with high probability the gap estimates $\widehat\Delta_{g,a}$ for $g\in\Gcal$ and $a\in\Acal_g$ are accurate enough, essentially up to constant factors. We also need to relate the estimated set of arms in contention $\Ccal(t)$ to the true arms in contention $\Ccal^\star(\epsilon;\Ical)$. 

\begin{lemma}\label{lemma:estimation_gaps}
    With probability at least $1-2/T^{2 c_\Gcal}$, the following holds. For any $t\in[T]$, any group $g\in\Gcal$ and action $a\in\Acal_g$,
    \begin{equation*}
        \frac{\Delta_{g,a}}{4} \1[\Delta_{g,a}\geq 3 \epsilon(t)]\leq \widehat \Delta_{g,a}(t) \leq 3\max(\Delta_{g,a}, \epsilon(t)).
    \end{equation*}
    Also, for all $t\in[T]$, $\Ccal(t) \subseteq \Ccal^\star(3\epsilon(t);\Ical)$.
\end{lemma}

Using these bounds, we can then lower bound the solution from $\Qcal(t)$ as follows:
\begin{align}\label{eq:lower_bound_q_t}
        q(t) &\geq \max_{x\geq 0,q\geq 0}  q  
                \textrm{ s.t.} \begin{cases} \sum_{a\in\Ccal(t)\cap \Acal_g}  3\max(\Delta_{g,a},  \epsilon(t))   x_{g,a} \leq  \epsilon(t),  &g\in\Gcal\\
                  \sum_{a\in\Ccal(t)\cap \Acal_g} x_{g,a} \leq 1, &g\in\Gcal\\
                  \sum_{g\in\Gcal} x_{g,a} \geq q, &a\in \Ccal^\star(3\epsilon(t);\Ical)
                \end{cases} \notag \\
                &\geq \frac{\epsilon(t)}{3} M(3\epsilon(t);\Ical).
    \end{align}
    Note that the quantity $q(t)$ lower bounds the exploration progress for each arm in contention. Indeed, denoting by $\Hcal_t$ the history before iteration $t$, for any $a\in\Ccal(t)$,
    \begin{equation}\label{eq:lower_bound_progress_t}
        \Ebb[P_a(t+1)-P_a(t)\mid\Hcal_t] \geq \sum_{g\in\Gcal} x_{g,a}(t) \geq q(t).
    \end{equation}
    We then proceed to lower bound the progress by epochs. Precisely, let $\epsilon_k:=2^{-k-1}$ for $k\geq 0$ and let $t_k=\min\{t\geq 1:\epsilon(t)\leq \epsilon_k\} = \min\{t\geq 1:P(t) \geq C\log T/\epsilon_k^2\}$. We recall that $P(t)$ is the minimum number of pulls of arms in contention. Further, during any epoch $[t_k,t_{k-1})$, using \cref{eq:lower_bound_q_t} we have $q(t)\geq \epsilon_{k+1}M(3\epsilon_k;\Ical)$. Together with \cref{eq:lower_bound_progress_t} and concentration bounds we obtain the following upper bound on the length of each epoch.

    \begin{lemma}\label{lemma:upper_bound_length_epoch}
        With probability at least $1-3/T^{2 c_\Gcal}$, the event from \cref{lemma:estimation_gaps} holds and for $k\geq 0$,
        \begin{equation*}
            \min(t_{k+1}-1,T) -t_k \leq \frac{2^6 C\log T }{ \epsilon_k^3 M(3\epsilon_k;\Ical)}.
        \end{equation*}
    \end{lemma}

    We next bound the regret separately on each epoch $[t_k,t_{k+1})$. Note that with some probability, \cref{alg:main} defaults to pulling the UCB arm. The regret for this part can be bounded using classical arguments. The main regret term corresponds to the regret incurred from following the allocation of arms $x(t)$ computed as the solution to $\Qcal(t)$. Under appropriate events, the corresponding immediate regret for group $g$ at iteration $t\in[t_k,t_{k+1})$ is
    \begin{equation*}
        \sum_{a\in\Acal_g\cap \Ccal(t)} \Delta_{g,a}x_{g,a}(t) \lesssim \sum_{a\in\Acal_g\cap \Ccal(t)} (\widehat\Delta_{g,a}+\epsilon(t)) x_{g,a}(t) \lesssim \epsilon(t) \leq \epsilon_k,
    \end{equation*}
    where in the first inequality we used the lower bound on $\widehat\Delta_{g,a}(t)$ from \cref{lemma:estimation_gaps} and in the second inequality we used the constraints of $\Qcal(t)$ on the solution $x(t)$. We then use concentration inequalities and sum the accumulated regret over each epoch, recalling that the length of each epoch is bounded via \cref{lemma:upper_bound_length_epoch}. After computations, we can obtain the desired regret bounds for \cref{thm:instance_dependent_regret}.

\section{Conclusion}
\label{sec:conclusion}

We introduced \ColUCB, an algorithm for minimizing collaborative min-max regret in grouped multi-armed bandit settings where groups share reward observations but have heterogeneous action sets. Specifically, \ColUCB coordinates groups to most efficiently explore arms in contention to minimize the worst-case regret over all groups. We show that our algorithm achieves optimal minimax as well as instance-dependent collaborative regret bounds up to logarithmic factors. In particular, our algorithm is adaptive to the structure of shared arms across groups. Altogether, these results provide new insights into when and how collaboration yields significant benefits in multi-agent learning. Notably, we focus on heterogeneity between groups in terms of the feasible action sets; we leave for future work investigating more complex heterogeneity models or more game theoretic interactions between groups.

\section*{Acknowledgements} MB was supported by a Columbia Data Science Institute postdoctoral fellowship.

\bibliographystyle{alpha}
\bibliography{refs}

\appendix

\section{Proof of the instance-dependent collaborative regret bound on \ColUCB}
\label{sec:instance_dependent_UB}

In this section, we prove \cref{thm:instance_dependent_regret} which gives an instance-dependent regret bound for \ColUCB. First, note that \ColUCB is scale-invariant, hence, it suffices to focus on the case when the subGaussian parameter is $\sigma^2=1$. Precisely, let $\Ical$ be an instance with value distributions subGaussian with parameter $\sigma^2$ and let $\Ical'$ be the instance where all value distributions have been rescaled via $x\mapsto x/\sigma$. Then, $\Ical'$ has subGaussian value distributions with parameter $1$. Suppose that the desired bound holds for $\sigma=1$. Then, we obtain
\begin{align}
    \MaxReg_T(\ColUCB;\Ical) &= \sigma \cdot \MaxReg_T(\ColUCB;\Ical)\notag \\
    &\leq \sigma\cdot c_0c_\Gcal \cdot(R(\epsilon;\Ical')\log T+t_{\min} \Delta_{\max}(\Ical')), \label{eq:bound_sigma_1}
\end{align}
where $\epsilon\in(0,1]$ is such that $T(\epsilon;\Ical')=T$. Next, we check that $\Delta_{\max}(\Ical)=\sigma\cdot \Delta_{\max}(\Ical')$, and that $M(\sigma \epsilon;\Ical) = \sigma\cdot M(\epsilon;\Ical')$ for any $\epsilon\in(0,1]$. Hence, we also have $T(\epsilon;\Ical,\sigma) = T(\epsilon;\Ical')$ and $R(\epsilon;\Ical,\sigma) = \sigma R(\epsilon;\Ical')$ for $\epsilon\in(0,1]$. Therefore, \cref{eq:bound_sigma_1} already yields the desired result.

We start by proving \cref{lemma:estimation_gaps} which bounds the estimation error for the gaps $\Delta_{g,a}$ and the set of arms in contention $\Ccal^\star(\epsilon;\Ical)$.

\vspace{3mm}

\begin{proof}[of \cref{lemma:estimation_gaps}] 
    For convenience, we introduce and recall the notations for any $a\in\Acal$ and $t\in[T]$,
    \begin{equation*}
        \epsilon_a(t):= \sqrt{\frac{C\log T}{P_a(t)}}\quad \text{and}\quad  \epsilon(t):= \sqrt{\frac{C\log T}{P(t)}}.
    \end{equation*}
    In particular, for any $a\in\Ccal(t)$ we have $\epsilon(t)\geq \epsilon_a(t)$.
    Last, we denote by $\Hcal_t$ the history available at the beginning of iteration $t\in[T]$. We introduce the good event
    \begin{equation*}
        \Ecal:=\set{ \forall a\in\Acal,\forall t\in[T], \abs{\hat \mu_a(t) - \mu_a} \leq \frac{\epsilon_a(t)}{3}},
    \end{equation*}
    which by Hoeffding's subGaussian inequality and the union bound has
    \begin{equation*}
        \Pbb[\Ecal^c] \leq \frac{2T|\Acal|}{T^{C/18}} \leq \frac{2}{T^{2 c_\Gcal}}.
    \end{equation*}
    We suppose from now on that $\Ecal$ holds. We also fix a group $g\in\Gcal$. For any $t\in[T]$, we have
    \begin{equation}\label{eq:good_arm_in_set_A}
        \max_{a'\in\Acal_g}\LCB_{a}(t) \overset{(i)}{\leq} \max_{a\in\Acal_g} \mu_{a} = \mu_g^\star \leq \UCB_{a_g^\star}(t),
    \end{equation}
    where in $(i)$ we used $\Ecal$. As a result, $a_g^\star\in\widehat\Acal_g(t)$. From now, we suppose that the event $\Ecal$ is satisfied. We now distinguish between the two cases for $\widehat\Acal_g(t)$.

    \paragraph{Case 1: There exists $t'\leq t$ with $|\widehat\Acal_g(t')|=1$.}
    In this case, we have $\widehat\Acal_g(t')=\{a_g^\star\}$ and as a result, $ \hat\mu_g^\star(t') = \hat\mu_{a_g^\star}(t') $. Hence, for any $a\in\Acal_g\setminus\{a_g^\star\}$, we have
    \begin{equation*}
        \mu_a \leq \UCB_a(t') - \frac{2\epsilon_a(t)}{3} \leq \LCB_{a_g^\star}(t')- \frac{2\epsilon_a(t)}{3} \leq \mu_g^\star - \frac{2(\epsilon_a(t') + \epsilon_{a_g^\star}(t'))}{3}.
    \end{equation*}
    In the last inequality we used $\Ecal$. Therefore, $\Delta_{g,a} \geq \frac{2}{3}(\epsilon_a(t') + \epsilon_{a_g^\star}(t')) \geq \frac{2}{3}(\epsilon_a(t) + \epsilon_{a_g^\star}(t))$, since $\epsilon_a(t)$ is non-increasing in $t$ for any $a\in\Acal$. Therefore, at time $t$, we have for any $a\in\Acal_g\setminus\{a_g^\star\}$,
    \begin{equation*}
        \hat \mu_a(t) \leq \mu_a +\frac{\epsilon_a(t)}{3} =  \mu_g^\star -\Delta_{g,a}+ \frac{\epsilon_a(t)}{3} \leq \hat \mu_{a_g^\star}(t) - \Delta_{g,a} + \frac{\epsilon_a(t)+\epsilon_{a_g^\star}(t)}{3} \leq \hat \mu_{a_g^\star}(t).
    \end{equation*}
    As a result, we also have $\hat \mu_g^\star(t) = \hat \mu_{a_g^\star}(t)$. This implies $\Delta_{g,a_g^\star}(t)=\Delta_{g,a_g^\star}=0$. Next, for any $a\in\Acal_g\setminus \{a_g^\star\}$,
    \begin{equation*}
        |\widehat\Delta_{g,a}(t)-\Delta_{g,a}| = |\hat\mu_g^\star(t) -\hat \mu_a(t) - (\mu_{a_g^\star}-\mu_a)| = |\hat\mu_{a_g^\star}(t) -\mu_{a_g^\star} +\mu_a-\hat \mu_a(t)| \leq \frac{\epsilon_a(t)+\epsilon_{a_g^\star}(t)}{3}.
    \end{equation*}
    Combining this with the lower bound $\Delta_{g,a}\geq \frac{2}{3} (\epsilon_a(t)+\epsilon_{a_g^\star}(t))$, we obtained
    \begin{equation}\label{eq:case1_estimate_Delta}
        \frac{\Delta_{g,a}}{2} \leq \widehat \Delta_{g,a}(t) \leq \frac{3\Delta_{g,a}}{2},\quad a\in\Acal_g.
    \end{equation}

    \paragraph{Case 2: For all $t'\leq t$, $|\widehat\Acal_g(t')|\geq 2$.} 
    In this case, we have $\Ccal_g(t')=\widehat\Acal_g(t')$ for all $t'\leq t$. Since we always have $a_g^\star\in\widehat\Acal_g(t')$ for $t'\leq t$ this shows that $a_g^\star\in\Ccal(t)$. In particular, $\epsilon_{a_g^\star}(t)\leq \epsilon(t)$. Next, let $\hat a_g^\star(t)\in\Acal_g$ be the maximizer for $\max_{a'\in\Acal_g}\hat\mu_{a'}(t)$. We first show that $\hat a_g^\star(t) \in\Ccal(t)$. From the assumption for case 2, it suffices to show that for all $t'\leq t$ we have $\hat a_g^\star(t) \in\widehat\Acal_g(t')$. To do so, we let $\tilde a(t')\in\Acal_g$ be the maximizer of $\max_{a'\in\Acal_g}\LCB_{a'}(t')$. We start by computing
    \begin{align*}
        \UCB_{\hat a_g^\star(t)}(t') \geq \mu_{\hat a_g^\star(t)} + \frac{2\epsilon_{\hat a_g^\star(t)}(t')}{3} 
        &\overset{(i)}{\geq} \hat \mu_{\hat a_g^\star(t)}(t) + \frac{\epsilon_{\hat a_g^\star(t)}(t)}{3}\\
        &\overset{(ii)}{\geq} \hat \mu_{\tilde a(t')}(t) + \frac{\epsilon_{\hat a_g^\star(t)}(t)}{3}\\
        &\geq \mu_{\tilde a(t')} + \frac{\epsilon_{\hat a_g^\star(t)}(t) -\epsilon_{\tilde a(t')}(t) }{3} \geq \LCB_{\tilde a(t')}(t').
    \end{align*}
    In $(i)$ and $(iii)$ we used the fact that $\epsilon_a(t)$ is non-increasing with $t$ for any $a\in\Acal$. In $(ii)$ we used the definition of $\hat a_g^\star(t)$. This shows that $\hat a_g^\star(t) \in\widehat\Acal_g(t')$ for all $t'\leq t$ and as a result we have $a_g^\star(t) \in\Ccal(t)$.
    Then,
    \begin{equation}\label{eq:useful_bound_1}
        \mu_g^\star \geq \mu_{\hat a_g^\star(t)} \geq \hat\mu_g^\star(t) -\frac{\epsilon_{\hat a_g^\star(t)}(t)}{3} \geq \hat \mu_{a_g^\star}(t)-\frac{\epsilon_{\hat a_g^\star(t)}(t)}{3} \geq \mu_g^\star -\frac{\epsilon_{\hat a_g^\star(t)}(t) + \epsilon_{a_g^\star}(t)}{3} \geq \mu_g^\star - \frac{2\epsilon(t)}{3}.
    \end{equation}
    Therefore, for any $a\in \Acal_g$, 
    \begin{align}
        |\widehat\Delta_{g,a}(t)-\Delta_{g,a}|&= |\hat\mu_g^\star(t) - \hat\mu_a(t) - \mu_g^\star + \mu_a | \notag\\
        &\leq |\hat\mu_g^\star(t) - \mu_{\hat a_g^\star(t)}| + |\mu_{\hat a_g^\star(t)} - \mu_g^\star| + |\mu_a - \hat\mu_a(t)| \notag\\
        &\leq \frac{\epsilon_{\hat a_g^\star(t)}(t)}{3} + \frac{2\epsilon(t)}{3} + \frac{\epsilon_a(t)}{3} \leq \epsilon(t) +\frac{\epsilon_a(t)}{3}. \label{eq:useful_relation_delta_widehat_delta}
    \end{align}
    Next, if $a\in\Acal_g\setminus\widehat\Acal_g$, using $\hat a_g^\star(t)\in\Ccal(t)$,
    \begin{equation*}
        \mu_a \leq \hat \mu_a -\frac{\epsilon_a(t)}{3} = \UCB_a(t) - \frac{4 \epsilon_a(t)}{3} \leq \max_{a'\in\Acal_g}\LCB_{a'}(t) - \frac{4 \epsilon_a(t)}{3} \leq \min(\hat \mu_g^\star(t), \mu_g^\star) - \frac{4 \epsilon_a(t)}{3}.
    \end{equation*}
    In the last inequality, we used \cref{eq:good_arm_in_set_A}.
    As a result, we have $\widehat\Delta_{g,a}(t)\geq \epsilon_a(t)$ and $\Delta_{g,a} \geq \frac{4}{3}\epsilon_a(t)$. Together with the previous equation, this implies $\Delta_{g,a}\geq \widehat\Delta_{g,a}(t) -\frac{1}{3}\epsilon_a(t)-\epsilon(t) \geq \frac{2}{3}\widehat\Delta_{g,a}(t) -\epsilon(t)$ and $\Delta_{g,a}\leq \frac{4}{3}\widehat \Delta_{g,a}(t)+\epsilon(t)$. 
    Rearranging, this gives
    \begin{equation}\label{eq:case2aa}
        \frac{3(\Delta_{g,a}-\epsilon(t))}{4}\leq \widehat\Delta_{g,a}(t) \leq \frac{3\Delta_{g,a}}{2} +\epsilon(t), \quad a\in\Acal_g\setminus \widehat\Acal_g(t).
    \end{equation}

    We next turn to bounding the estimation error in the gaps for $a\in\widehat\Acal_g(t)$. Let $\tilde a\in\Acal_g$ be the maximizer for $\max_{a'\in\Acal_g}\LCB_a(t)$. Then,
    \begin{equation*}
        \mu_g^\star - \Delta_{g,\tilde a}=\mu_{\tilde a} \geq \LCB_{\tilde a}(t) + \frac{2\epsilon_{\tilde a}(t)}{3} \geq \LCB_{a_g^\star}(t) + \frac{2\epsilon_{\tilde a}(t)}{3} \geq \mu_g^\star + \frac{2\epsilon_{\tilde a}(t) - 4 \epsilon_{a_g^\star}(t)}{3} .
    \end{equation*}
    Rearranging, this gives $\Delta_{g,\tilde a}\leq \frac{4}{3}\epsilon_{a_g^\star}(t)\leq \frac{4}{3}\epsilon(t)$ and similarly, $\epsilon_{\tilde a}(t) \leq 2\epsilon(t)$. We now fix $a\in\widehat\Acal_g(t)$. We have
    \begin{equation*}
        \mu_a \geq \hat\mu_a  - \frac{\epsilon_a(t)}{3} \geq \UCB_a(t) - \frac{4\epsilon_a(t)}{3} \geq \max_{a'\in\Acal_g}\LCB_{a'}(t) -\frac{4\epsilon_a(t)}{3} \geq \max(\hat\mu_g^\star(t), \mu_g^\star) -\frac{4(\epsilon_a(t)+\epsilon(t))}{3},
    \end{equation*}
    where in the last inequality we lower bounded the maximum over all $a'\in\Acal_g$ either using $a'=a_g^\star\in\Ccal(t)$ or $a'=\hat a_g^\star(t)\in\Ccal(t)$. In particular, if $a\in\Ccal(t)$ we also have $\epsilon_a(t)\leq \epsilon(t)$ which gives
    \begin{equation}\label{eq:case2ab}
        \Delta_{g,a},\widehat\Delta_{g,a}(t) \leq \frac{8\epsilon(t)}{3} ,\quad a\in\widehat\Acal_g(t) \cap\Ccal(t).
    \end{equation}

    If $a\notin\Ccal(t)$ this means that there exists $t'\leq t$ such that $a\notin \widehat\Acal_g(t')$ (recall that from the case 2 assumption $\widehat\Acal_g(t')=\Ccal_g(t')$ for all $t'\leq t$). The previous bounds then apply, which implies that $\Delta_{g,a} \geq \frac{4}{3} \epsilon_a(t') \geq \frac{4}{3}\epsilon_a(t)$. Combining this with \cref{eq:useful_relation_delta_widehat_delta} shows that $\widehat\Delta_{g,a}(t) \geq \epsilon_a(t) -\epsilon(t)$ and $\widehat\Delta \leq $
    \begin{equation}\label{eq:case2b}
        \frac{3}{4}\Delta_{g,a} -\epsilon(t)\leq \widehat\Delta_{g,a}(t)  \leq \frac{5}{4}\Delta_{g,a} + \epsilon(t),\quad a\in\widehat\Acal_g(t)\setminus\Ccal(t).
    \end{equation}

    In summary, we can combine \cref{eq:case1_estimate_Delta,eq:case2aa,eq:case2ab,eq:case2b} which shows that in all cases,
    \begin{equation*}
        \frac{3\Delta_{g,a}}{8} \1\sqb{\Delta_{g,a} \geq\frac{8\epsilon(t)}{3}}  \leq \widehat\Delta_{g,a}\leq \frac{3\Delta_{g,a}}{2} + \frac{8\epsilon(t)}{3},  \quad a\in\Acal_g.
    \end{equation*}
    Up to minor simplifications, this proves the first claim under $\Ecal$.

    Next, for any $a\in\Ccal(t)$, let $g\in\Gcal$ such that $a\in\Ccal_g(t')=\widehat\Acal_g(t')$ for all $t'\leq t$. From \cref{eq:case2ab} in this case we already obtained $\Delta_{g,a} \leq \frac{8}{3}\epsilon(t)$. This implies that $a\in \Ccal^\star(8\epsilon(t)/3;\Ical)$ under $\Ecal$. This ends the proof.
\end{proof}

In the rest of the proof, we denote by $\Ecal$ the event on which the guarantees from \cref{lemma:estimation_gaps} hold and suppose that this event holds from now.
We now lower bound the value $q(t)$ of the linear program $\Qcal(t)$ solved at time $t$ by \ColUCB. We have the following where we use \cref{lemma:estimation_gaps} in the first inequality:
    \begin{align}
        q(t) &\geq \max_{x\geq 0,q\geq 0}  q \quad
                \textrm{s.t.} \begin{cases} \sum_{a\in\Ccal(t)\cap \Acal_g}  \max(\Delta_{g,a},\epsilon(t)) \cdot  x_{g,a} \leq \frac{\epsilon(t)}{3},  &g\in\Gcal\\
                  \sum_{a\in\Ccal(t)\cap \Acal_g} x_{g,a} \leq 1, &g\in\Gcal\\
                  \sum_{g\in\Gcal:a\in\Acal_G} x_{g,a} \geq q, &a\in \Ccal(3\epsilon(t);\Ical)
                \end{cases} \notag \\
            &= \max_{x\geq 0,q\geq 0}  q \quad
                \textrm{s.t.} \begin{cases} \sum_{a\in\Ccal(t)\cap \Acal_g}  \max(\Delta_{g,a}, \epsilon(t)) \cdot  x_{g,a} \leq \frac{\epsilon(t)}{3},  &g\in\Gcal\\
                  \sum_{g\in\Gcal:a\in\Acal_g} x_{g,a} \geq q, &a\in \Ccal(3\epsilon(t);\Ical)
                \end{cases} \notag \\
            &= \frac{\epsilon(t)}{3} M(3\epsilon(t);\Ical). \label{eq:lower_bound_q}
    \end{align}
    In the first equality, we noted that the first constraint already implies $\sum_{a\in\Ccal(t)\cap \Acal_g}   x_{g,a} \leq \frac{1}{3}$. 
    Using this bound, we can lower bound the growth of the number of pulls of each arm in contention. Precisely, by construction of \ColUCB, for any $a\in\Ccal(t)$ we have
    \begin{equation}\label{eq:progress_nb_pulls}
        \Ebb[P_a(t+1)-P_a(t)\mid\Hcal_t] \geq  \sum_{g\in\Gcal}\Pbb[a_g(t) = a\mid\Hcal_t] \geq \sum_{g\in\Gcal} x_{g,a}(t) \geq q(t).
    \end{equation}
    We use this to upper bound the length of each epoch, as detailed in \cref{lemma:upper_bound_length_epoch}. We recall the corresponding notation. For $k\geq 0$ 
    \begin{equation*}
        \epsilon_k:=\frac{1}{2^{k+1}} \quad \text{and} \quad
        t_k := \min\set{t\in[T]: \epsilon(t)\leq \epsilon_k} = \min\set{t\in[T]: P(t)\geq \frac{C\log T}{\epsilon_k^2} }.
    \end{equation*}
    We also denote $t_0=1$ and $\epsilon_0=\infty$.

    \vspace{3mm}

\begin{proof}[of \cref{lemma:upper_bound_length_epoch}]
    We can bound the concentration of the number of pulls $P_a(t)$ for $a\in\Acal$ and $t\in[T]$ using Freedman's inequality. For any $1\leq t_1\leq t_2$ and $a\in\Acal$, conditionally in $\Hcal_{t_1}$, with probability at least $1-\delta$, we have
    \begin{align}
         P_a(t_2) -P_a(t_1) &\overset{(i)}{\geq} \sum_{t=t_1}^{t_2-1}q(t) \1[a\in\Ccal(t)] -\frac{1}{2}\sum_{t=t_1}^{t_2-1} \sum_{g\in\Gcal}\Var[\1[a_g(t)=a]\1[a\in\Ccal(t)]\mid \Hcal_t]  - 2\log\frac{1}{\delta} \notag\\
         &\geq \frac{1}{2} \sum_{t=t_1}^{t_2-1}q(t) \1[a\in\Ccal(t)]  - 2\log\frac{1}{\delta}. \label{eq:lower_bound_progress_pulls}
    \end{align}
    In $(i)$ we used Freedman's inequality (\cref{thm:freedman_inequality}) to the martingale increments $\1[a_g(t)=a]\1[a\in\Ccal(t)] - \Pbb[a_g(t)=a,a\in\Ccal(t)\mid\Hcal_t]$ together with \cref{eq:progress_nb_pulls}. We then define the event
    \begin{equation*}
        \Fcal:=\set{\forall 1\leq t_1\leq t_2\leq T, \forall a\in\Acal, P_a(t_2)-P_a(t_1) \geq \frac{1}{2} \sum_{t=t_1}^{t_2-1}q(t) \1[a\in\Ccal(t)]  - C\log T },
    \end{equation*}
    which from \cref{eq:lower_bound_progress_pulls} has probability at least $1-|\Acal|/T^{C/2-2}\geq 1-T^{-2 c_\Gcal}$.
    We note that since $\Ccal(t)$ is non-increasing, $P(t)$ is non-decreasing and $\epsilon(t)$ is non-increasing.
    We now bound the length $t_{k+1}-t_k$ for $k\geq 0$ such that $t_{k+1}>t_k$ and $T\geq t_k$, otherwise the desired result is immediate. Let $t_k':=\min(t_{k+1}-1,T)$.
    Let $\Ecal$ be the event on which \cref{lemma:estimation_gaps} holds. We recall that on $\Ecal$, \cref{eq:lower_bound_q} holds.
    Then, under $\Ecal\cap\Fcal$, for any $a\in\Acal$ and $k\geq 1$, we have
    \begin{align*}
        P_a(t_k')&\overset{(i)}{\geq} \frac{1}{6} \sum_{t=t_k}^{t_k'-1} \epsilon(t) M( 3\epsilon(t) ;\Ical) \1[a\in\Ccal(t)]  - C\log T \\
        &\overset{(ii)}{\geq} \frac{\epsilon_{k+1}}{6}   M( 3\epsilon_k ;\Ical) \cdot \abs{\set{t\in[t_k,t_k'): a\in\Ccal(t)}} - C\log T\\
        &\overset{(iii)}{\geq} \frac{t_k'-t_k}{12}\epsilon_{k} M(3\epsilon_k;\Ical) \1[a\in\Ccal(t_k')] - C\log T.
    \end{align*}
    In $(i)$ we used \cref{eq:lower_bound_progress_pulls,eq:lower_bound_q} and in $(ii)$ we used the definition of $t_{k+1}$ and the fact that $M(\cdot;\Ical)$ is non-increasing and $\epsilon(t)$ is non-increasing in $t$. Furthering the previous bound, we obtained for $k\geq 1$,
    \begin{align*}
        \frac{C\log T}{\epsilon_{k+1}^2} > P(t_k') =\min_{a\in\Ccal(t_k')} P_a(t_k')\geq \frac{t_k'-t_k}{12}\epsilon_{k} M(3\epsilon_k;\Ical) - C\log T.
    \end{align*}
    Rearranging, this gives the desired bound
    \begin{equation*}
        t_k'-t_k \leq \frac{2^6C\log T}{\epsilon_k^3 M(3\epsilon_k;\Ical)}.
    \end{equation*}
    This ends the proof.
\end{proof}

\cref{lemma:upper_bound_length_epoch} bounds the lengths of all epochs $[t_k,t_{k+1})$ for $k\geq 1$ under some event which we call $\Fcal$. We recall that by construction after $16C \log T \cdot t_{\min}$ iterations, each arm has been pulled at least $16C\log T$ times and hence $\epsilon(t)\leq 1/4$. Hence, $t_1-1 \leq 16C\log T\cdot t_{\min}$.
We now define $k_{\max}$ the integer $k\geq 0$ such that $\epsilon(T)\in(\epsilon_{k+1},\epsilon_{k}]$. We sum the previous equations to obtain
    \begin{align*}
        T-k_{\max}-16C\log T\cdot t_{\min} &\leq 2^6 C\log T\cdot \sum_{k=1}^{k_{\max}} \frac{1}{\epsilon_k^3 M(3\epsilon_k;\Ical)} \\
        &\leq 2^9 C\log T \int_{3\epsilon_{k_{\max}+1}}^{1} \frac{d\epsilon}{\epsilon^4 M(\epsilon;\Ical)}  \\
        &\overset{(i)}{\leq} c_0 C\log T \paren{\int_{\epsilon(T)}^{1} \frac{d\epsilon}{\epsilon^4 M(\epsilon;\Ical)} } = c_0 C\log T  \cdot T(\epsilon(T);\Ical),
    \end{align*}
    for some universal constant $c_0>0$. 
    In $(i)$ we used the fact that $ M(\epsilon;\Ical)$ is non-increasing in $\epsilon$.
    Recall that by continuity of $T(\cdot;\Ical)$ (see \cref{lemma:continuity_R_T}) we have $T(\epsilon(T);\Ical)=T$. 
    Next, note that we always have $M(\epsilon;\Ical)\leq \frac{|\Gcal|}{\epsilon}$ and hence $\epsilon(T) \gtrsim 1/\sqrt{T|\Gcal|}$. In particular we have $k_{\max}=\Ocal(C\log T)$.
    In summary, we showed that under the assumption for $T$, there exists a constant $c_1\geq 1$ such that under $\Fcal$, 
    \begin{equation}\label{eq:bound_tilde_epsilon}
        \epsilon(T) \leq \min\set{z\in(0,1]: T(z;\Ical) \leq \frac{T}{c_1C\log T}}:=\eta(T) .
    \end{equation}

    It remains to bound the regret of \ColUCB in terms of the function $R(\cdot;\Ical)$ and $\eta(T)$. 
    For any fixed group $g\in\Gcal$ and $t> t_{\min}$, we denote by $I_g(t)$ the indicator whether $a_g(t)$ was sampled according to the solution $x(t)$ of $\Qcal(t)$ (which occurs with probability $\sum_{a\in\Acal_g\cap\Ccal(t)} x_{g,a}(t)$ when $P(t)>0$.
    
    Suppose that $I_g(t)=0$ for some fixed time $t$. By construction of \ColUCB we then have $a_g(t)=\hat a_g^\star(t)$ such that $\widehat\Delta_{g,a_g(t)}=0$. From \cref{lemma:estimation_gaps}, under the event $\Fcal$ this implies $\Delta_{g,a_g(t)} \leq 3\epsilon(t)$. Of course, we also have $\Delta_{g,a_g(t)}\leq \Delta_{\max}$. For convenience, we introduce the notation
    \begin{equation*}
        \tilde\epsilon(t) =\min(\epsilon(t),\Delta_{\max}).
    \end{equation*}
    In summary, we proved that under $\Ecal$,
    \begin{equation}\label{eq:regret_bound_UCB_part}
        \sum_{t=t_1}^T (1-I_g(t)) \cdot \Delta_{g,a_g(t)} \leq 2\sum_{t=t_1}^T \tilde \epsilon(t).
    \end{equation}
    Hence, to bound the regret of group $g$, it suffices to focus on times for which $I_g(t)=1$. By construction, we have
    \begin{equation*}
        \Ebb[I_g(t)\Delta_{g,a_g(t)}\mid\Hcal_t] = \sum_{a\in\Acal\cap\Ccal(t)} \Delta_{g,a}x_{g,a}(t).
    \end{equation*}
    Hence, we apply Freedman's inequality \cref{thm:freedman_inequality} which shows that with probability at least $1-\delta$,
    \begin{align*}
        \sum_{t=t_1}^T I_g(t)\Delta_{g,a_g(t)} &\leq \sum_{t=t_1}^T \sum_{a\in\Acal_g\cap\Ccal(t)} \Delta_{g,a}x_{g,a}(t) + \sum_{t=t_1}^T \Var[I_g(t)\Delta_{g,a_g(t)}\mid\Hcal_t] + 2\Delta_{\max}\log\frac{1}{\delta}\\
        &\leq 2\sum_{t=t_1}^T \sum_{a\in\Acal_g\cap\Ccal(t)} \Delta_{g,a}x_{g,a}(t) + 2\Delta_{\max}\log\frac{1}{\delta}.
    \end{align*}
    We then define the event for $g\in\Gcal$:
    \begin{equation*}
        \Kcal:= \set{\forall g\in\Gcal: \sum_{t=t_1}^T I_g(t)\Delta_{g,a_g(t)} \leq 2\sum_{t=t_1}^T \sum_{a\in\Acal_g\cap\Ccal(t)} \Delta_{g,a}x_{g,a}(t) + C \Delta_{\max} \log T },
    \end{equation*}
    which by the union bound has probability at least $1-|\Gcal|T^{-C}\geq 1-T^{-2 c_\Gcal}$. Then, under $\Fcal\cap\Kcal$, we obtained
    \begin{align*}
        \sum_{t=t_1}^T I_g(t) \Delta_{g,a_g(t)}
        &\overset{(i)}{\leq} 2\sum_{t=t_1}^T \sum_{a\in\Acal\cap\Ccal(t)} \Delta_{g,a}x_{g,a}(t) + C\Delta_{\max}\log T\\
        &\overset{(ii)}{\leq} 2\sum_{t=t_1}^T \sum_{a\in\Acal\cap\Ccal(t)} (4\widehat\Delta_{g,a} + 3 \tilde \epsilon(t))x_{g,a}(t) + C\Delta_{\max}\log T\\
        &\overset{(iii)}{\leq}18\sum_{t=t_1}^T  \tilde \epsilon(t) + C\Delta_{\max}\log T.
    \end{align*}
    In $(i)$ we used $\Kcal$ and in $(ii)$ we used the lower bound on $\widehat\Delta_{g,a}$ from \cref{lemma:estimation_gaps} together with $\Delta_{g,a}\leq \Delta_{\max}$. In $(iii)$ we used the definition of the linear program $\Qcal(t)$. Together with \cref{eq:regret_bound_UCB_part}, this implies that under $\Fcal\cap\Kcal$, for any $g\in\Gcal$,
    \begin{equation}\label{eq:first_max_regret_bound}
        \Reg_{g,T}(\ColUCB;\Ical) =\sum_{t=1}^T\Delta_{g,a_g(t)} \leq 20\sum_{t=t_1}^T \tilde\epsilon(t) + 2 t_{\min}\Delta_{\max},
    \end{equation}
    where in the last inequality we used $t_1-1\leq 16C\log T \cdot t_{\min}$.    
    To further bound \cref{eq:first_max_regret_bound}, we decompose the sum on the right-hand side as follows
    \begin{equation*}
        \sum_{t=t_1}^T \tilde\epsilon(t) \leq \sum_{t=t_1}^T  \tilde\epsilon(t) \1[ \tilde\epsilon(t)\geq \eta(T)] + T\eta(T).
    \end{equation*}
    Note that for any $\epsilon\in(0,1]$,
    \begin{equation*}
        R(\epsilon;\Ical)=\epsilon \int_\epsilon^1 \frac{ dz}{M(z;\Ical)z^3\epsilon} \geq \epsilon \int_\epsilon^1 \frac{ dz}{M(z;\Ical)z^4} =\epsilon T(\epsilon;\Ical),
    \end{equation*}
    hence, we obtain the decomposition
    \begin{equation}\label{eq:second_decomposition}
        \sum_{t=t_1}^T \tilde\epsilon(t) \leq \sum_{t=t_1}^T  \tilde\epsilon(t) \1[ \tilde\epsilon(t)\geq \eta(T)] + c_2 C \log T\cdot R(\eta(T);\Ical).
    \end{equation}

    Then, let $l_{max}$ be the integer $l\geq 0$ such that $\eta(T)\in(\epsilon_{l+1},\epsilon_l]$. Under $\Fcal\cap\Kcal$, denoting $t_k':=\min(t_{k+1}-1,T)$ for $k\geq 1$, we have
    \begin{align*}
        \sum_{t=t_1}^T  \tilde\epsilon(t) \1[ \tilde\epsilon(t)\geq \eta(T)] &\leq   \sum_{k=1}^{l_{max}} (t_k'-t_k+1)\min(\epsilon_k,\Delta_{\max})\\
        &\overset{(i)}{\leq} l_{max}\Delta_{\max}+2^6 C\log T\cdot  \sum_{k=1}^{l_{max}} \frac{1}{\epsilon_k^2 M(3\epsilon_k;\Ical)} \\
        &\leq l_{max}\Delta_{\max}+ c_3 C\log T \int_{\eta(T)}^{1} \frac{d\epsilon}{\epsilon^3 M(\epsilon;\Ical)} \\
        &\overset{(ii)}{\leq} c_4 C\log T \paren{\Delta_{\max}+ R(\eta(T);\Ical)} ,
    \end{align*}
    for some universal constants $c_3,c_4>0$. In $(i)$ we used \cref{lemma:upper_bound_length_epoch} and the fact that $l_{max}\leq k_{max}$ since $\eta(T)\geq \epsilon(T)$. In $(ii)$ we also used $k_{max}=\Ocal(C\log T)$. 
    Together with \cref{eq:first_max_regret_bound,eq:second_decomposition}, we proved that under $\Fcal\cap\Kcal$, for some universal constant $c_5>0$,
    \begin{equation*}
        \max_{g\in\Gcal}\Reg_{g,T}(\ColUCB;\Ical) \leq  c_5 C\cdot( R(\epsilon(T);\Ical)\log T + t_{\min}\Delta_{\max}).
    \end{equation*}
    where we also used $R(\epsilon(T);\Ical)\geq R(\eta(T);\Ical)$ since $\epsilon(T)\leq\eta(T)$.
    Taking the expectation and recalling that $\Fcal\cap\Kcal$ has probability at least $1-4/T^{-2 c_\Gcal}$ gives the desired regret bound.

\section{Proofs of instance-dependent lower bounds on the collaborative regret}
\label{sec:instance_dependent_LB}

In this section, we present our instance-dependent lower bounds, starting with the case of Gaussian reward distributions.

\vspace{3mm}

\begin{proof}[of \cref{thm:instance_lower_bound_v2}]
Fix an instance $\Ical$ with reward distribution variance $1$, and a horizon $T\geq 2$. For convenience, we write
$\epsilon_T(\Ical) := \min\set{z\in(0,1]: T(z;\Ical) \leq T }.$
    Note as discussed in the proof of \cref{thm:instance_dependent_regret}, we always have $M(\epsilon;\Ical) \leq |\Gcal|/\epsilon$ and hence $\epsilon_T(\Ical) \gtrsim  1/\sqrt{|\Gcal|T}$.
    Next, we recall that $M(\epsilon;\Ical)$ is non-decreasing as $\epsilon$ decreases to $0$. 
    We let
    \begin{equation*}
        z_T(\Ical) \in \argmin_{z\in[\epsilon_T(\Ical),\frac{1+\epsilon_T(\Ical)}{2}]}  M(z;\Ical)z^2.
    \end{equation*}
    For convenience, we drop all arguments $\Ical$ when it is clear from the context which instance we refer to. Then,
    \begin{equation*}
        R(\epsilon_T  ) \geq \int_{z_T }^1 \frac{dz}{M(z_T )z^3} \geq \frac{1-z_T}{2M(z_T ) z_T^2} \geq \frac{1-\epsilon_T}{4M(z_T ) z_T^2} 
    \end{equation*}
    On the other hand,
    \begin{equation*}
        R(\epsilon_T)\leq \int_{\epsilon_T}^1 \frac{4dz}{M(z_T)z_T^2\cdot z} = \frac{4\log(1/\epsilon_T)}{M(z_T)z_T^2} \overset{(i)}{\leq} c_0 \frac{(1-\epsilon_T) \log T }{M(z_T)z_T^2},
    \end{equation*}
    for some universal constant $c_0>0$,
    where in $(i)$ we used $\epsilon_T(\Ical) \gtrsim  1/\sqrt{|\Gcal|T}$. 
    Last, using the same arguments as when lower bounding $R(\epsilon_T)$ we have
    \begin{equation*}\label{eq:lower_bound_T}
        T\overset{(i)}{=}T(\epsilon_T)\geq T(z_T) \geq \frac{1-\epsilon_T}{4M(z_T) z_T^3}.
    \end{equation*}
    In $(i)$ we used \cref{lemma:continuity_R_T}.
    For convenience, let $\bar t_0:=\frac{1-\epsilon_T}{100M(z_T) z_T^3}$ and let $t_0\in\set{\floor{\bar t_0},\ceil{\bar t_0}}$ be a random variable independent from all other variables such that $\Ebb[t_0]=\bar t_0$.
    In particular, we always have $T\geq t_0$. 

    With these preliminaries at hand, we now fix any algorithm {\sf ALG} and suppose that it achieves the following collaborative regret bound on the instance $\Ical$:
    \begin{equation}\label{eq:upper_bound_MaxReg_assumption}
        \MaxReg_T({\sf ALG};\Ical) \leq \frac{R(\epsilon_T)}{200c_0\log T} \leq \frac{1-\epsilon_T}{200M(z_T)z_T(\Ical)^2}.
    \end{equation}

    We aim to construct another instance $\Jcal$ on which {\sf ALG} incurs large collaborative regret. We use the instance $\Ical$ as the baseline and for readability, we introduce some simplifying notations for that instance. Precisely, we write $\Ccal^\star:=\Ccal^\star(z_T(\Ical);\Ical)$ and we keep writing $M(z_T)$ for $M(z_T(\Ical);\Ical)$. Next, we write $\mu_a^{(0)}$ for the value $\mu_a$ of arm $a\in\Acal$ in instance $\Ical$, and similarly, we write $\mu_g^{(0),\star}$ for $\mu_g^\star(\Ical)$. 
    Next, we let $S_0:=\set{g\in\Gcal: \Delta_{g,\min}(\Ical)\leq z_T(\Ical)}$ be the set of groups that have $z_T(\Ical)$-undifferentiated arms within instance $\Ical$. 
    By definition of $M(z_T)$, note that for any realization of the learning strategy {\sf ALG} under instance $\Ical$,
    \begin{align*}
        \max_{a\in\Ccal^\star} \sum_{g\in\Gcal} \sum_{t=1}^{t_0} \1[a_g(t)=a] &\leq M(z_T)\cdot \max_{g\in\Gcal} \paren{ \sum_{a\in\Acal_g} \max(\Delta_{g,a}(\Ical),z_T) \cdot\sum_{t=1}^{t_0} \1[a_g(t)=a] }\\
        &\leq M(z_T) \cdot \max_{g\in\Gcal}\paren{z_T  t_0 + \sum_{t=1}^T \Delta_{g,a_g(t)}(\Ical)}\\
        &= M(z_T) \paren{ z_T t_0 + \max_{g\in\Gcal} \Reg_{g,T}({\sf ALG};\Ical)}.
    \end{align*}
    Taking the expectation, we obtained
    \begin{align*}
        \max_{a\in\Ccal^\star} \Ebb[P_a(t_0)] \leq \Ebb\sqb{\max_{a\in\Ccal^\star} P_a(t_0) } 
        \leq M(z_T)\paren{z_T \bar t_0 + \MaxReg_T({\sf ALG};\Ical)} \leq \frac{1-\epsilon_T}{64z_T^2},
    \end{align*}
    where in the last inequality we used \cref{eq:upper_bound_MaxReg_assumption} and the definition of $\bar t_0$.
    We then fix $a_0\in\Ccal^\star$ for which
    
    \begin{equation}\label{eq:lower_bound_pull_a_0}
        \Ebb[P_{a_0}(t_0)] \leq \frac{1-\epsilon_T}{(8z_T)^2}.
    \end{equation}
    By definition of $\Ccal^\star$ there exists $g_0\in S_0$ such that $\Delta_{g_0,a_0}(\Ical)\leq z_T$.
    
    We now introduce two alternative instances $\Jcal_+$ and $\Jcal_-$ in which the only modification to $\Ical$ is in the reward of arm $a_0$. Precisely, the arm values are unchanged $\mu_a(\Jcal_+)=\mu_a(\Jcal_-):=\mu_a^{(0)}$ for $a\in\Acal\setminus\{a_0\}$. Letting $\nu_{g_0}=\argmax_{a\in\Acal_{g_0}\setminus\{a_0\}}\mu_a(\Ical)$ be the second-best arm value and we let
    \begin{equation*}
        \mu_{a_0}(\Jcal_+):= \nu_{g_0} + z_T \quad\text{and} \quad \mu_{a_0}(\Jcal_-):= \nu_{g_0} - z_T.
    \end{equation*}
    Since $g_0\in S_0$, we have $\Delta_{g_0,\min}\leq z_T$ and hence $|\mu_{a_0}^{(0)}-\nu_{g_0}|\leq z_T$. Therefore, we have $|\mu_{a_0}^{(0)}-\mu_{a_0}(\Jcal_+)|\leq 2z_T$ and similarly $|\mu_{a_0}^{(0)}-\mu_{a_0}(\Jcal_-)|\leq 2z_T$. Consider the first $P=\floor{(4z_T)^{-2}}$ pulls of $a_0$ in $\Ical$ and in $\Jcal_+$ (resp. $\mu_{a_0}(\Jcal_-)$). Then, the total variation between these pulls from $\Ical$ and $\Jcal_+$ (resp. $\Jcal_-$) is at most
    \begin{equation}\label{eq:bounding_TV_distance_gaussian}
        \TV(\Ncal(0,1)^{\otimes P},\Ncal(2z_T,1)^{\otimes P}) \leq  \sqrt{\frac{P}{2}\Dkl(\Ncal(2z_T,1)\parallel \Ncal(0,1) )} \leq z_T\sqrt{P} \leq \frac{1}{4},
    \end{equation}
    where in the first inequality we used Pinsker's inequality. As a result, there is a coupling between the pulls of $\Ical,\Jcal_+,\Jcal_-$ such that on an event $\Ecal(P)$ of probability at least $1/2$, the rewards of the pulls from all arms in $\Acal\setminus\{a_0\}$ coincide and the first $P$ pulls of $a_0$ coincide. Notice that under the event $\Fcal:=\Ecal\cap\{P_{a_0}(t_0)\leq P\}$ the run of {\sf ALG} on these three instances is identical for the first $t_0$ rounds. 
    Then, using \cref{eq:lower_bound_pull_a_0} and Markov's inequality, and we have
    \begin{align*}
        \Pbb(\Fcal) \geq \frac{1}{2}-\Pbb_{\Ical}(P_{a_0}(t_0)> P) \geq \frac{1}{4}.
    \end{align*}
    
    Under $\Jcal_+$ (resp. $\Jcal_-$), group $g_0$ incurs a regret $z_T$ regret for each iteration that it does not pull arm $a_0$ (resp. pulls arm $a_0$). Formally,
    \begin{equation*}
        \Reg_{g_0,t_0}({\sf ALG};\Jcal_+) \geq z_T\paren{t_0 -  \sum_{t=1}^{t_0}\1[a_{g_0}(t)=a_0]}  \text{  and  } \Reg_{g_0,t_0}({\sf ALG};\Jcal_-) \geq z_T \sum_{t=1}^{t_0}\1[a_{g_0}(t)=a_0].
    \end{equation*}
    Since the run of {\sf ALG} on the first $t_0$ iterations of $\Jcal_+$ and $\Jcal_-$ coincide under $\Fcal$, we can sum the previous equations under $\Fcal$ to obtain
    \begin{equation*}
        \Ebb_{\Jcal_+}\sqb{\Reg_{g_0,t_0}({\sf ALG};\Jcal_+)} + \Ebb_{\Jcal_-}\sqb{\Reg_{g_0,t_0}({\sf ALG};\Jcal_-)} \geq \Pbb(\Fcal)\cdot z_T \bar t_0\geq \frac{z_T \bar t_0}{4}.
    \end{equation*}
    As a result, under at least one instance from $\Jcal_+$ or $\Jcal_-$, {\sf ALG} incurs a collaborative regret of at least $z_T \bar t_0/8$. In summary, we proved that there exists some other instance $\Jcal$ with
    \begin{equation*}
        \MaxReg_T({\sf ALG};\Jcal) \geq \Ebb_{\Jcal}\sqb{\Reg_{g_0,T}({\sf ALG};\Jcal) } \geq \frac{z_T\bar t_0}{8} \geq \frac{1}{2^{10} M(z_T)z_T^2} \geq c_1 \frac{R(\epsilon_T)}{\log T},
    \end{equation*}
    for some universal constant $c_1>0$.
\end{proof}

As a corollary, we can also obtain a counterpart of \cref{thm:instance_lower_bound_v2} for Bernoulli reward distributions. The following instance-dependent bounds focus on the simpler case where mean rewards are bounded away from $0$ or $1$ (note that if mean rewards are close to the boundaries $0$ or $1$, optimal algorithms should incorporate KL divergence information, similar to KL-UCB).

\begin{corollary}\label{cor:instance_dependent_LB_bernoulli}
    There exists a universal constants $c_0,c_1>0$ such that the following holds. Fix a either the Gaussian or Bernoulli model $\Mcal\in\{\Mcal_G,\Mcal_G\}$. Let {\sf ALG} be an algorithm for the groups and $T\geq 2$. Suppose that for some Bernoulli instance $\Ical\in\Mcal$ with mean rewards in $[1/4,3/4]$, its collaborative regret satisfies
    \begin{equation*}
        \MaxReg_T({\sf ALG};\Ical)\leq c_0 \frac{R(\epsilon;\Ical)}{\log T}
    \end{equation*}
    where $\epsilon\in(0,1]$ satisfies $T(\epsilon;\Ical)= T$. Then, there exists another instance $\Ical'\in\Mcal$ obtained by modifying the mean value of a single arm, for which
    \begin{equation*}
        \MaxReg_T({\sf ALG};\Ical')\geq c_1 \frac{R(\epsilon;\Ical)}{\log T}
    \end{equation*}
\end{corollary}

\begin{proof}
    We only need to make minor modifications to the proof of \cref{thm:instance_lower_bound_v2}. We use the same notations therein. Precisely, fix a model $\Mcal\in\{\Mcal_G,\Mcal_B\}$ and an instance $\Ical\in\Mcal$ satisfying the assumptions and let {\sf ALG} be an algorithm which achieves the collaborative regret bound
    \begin{equation*}
        \MaxReg_T({\sf ALG};\Ical) \leq \frac{R(\epsilon_T(\Ical);\Ical)}{4^2\cdot 200c_0\log T},
    \end{equation*}
    similarly as in \cref{eq:upper_bound_MaxReg_assumption}.
    We define $\epsilon_T,z_T$ as in the proof of \cref{thm:instance_lower_bound_v2} and let $\bar t_0:=\frac{1-\epsilon_T}{4^2\cdot 100M(z_T)z_T^3}$. From there, we define $t_0$ as in the original proof and show that there exists $a_0\in\Ccal^\star$ for which 
    \begin{equation*}
        \Ebb[P_{a_0}(t_0)]\leq \frac{1-\epsilon_T}{(32z_T)^2},
    \end{equation*}
    similarly as in \cref{eq:lower_bound_pull_a_0}. We define $g_0$ and $\nu_{g_0}$ as in the original proof. We next adjust the alternatives instances $\Jcal_+$ and $\Jcal_-$ compared to the proof of \cref{thm:instance_lower_bound_v2}.  Specifically, all arms values are still unchanged $\mu_a(\Jcal_+)=\mu_a(\Jcal_-):=\mu_a^{(0)}$ for $a\in\Acal\setminus\{a_0\}$, but for arm $a_0$ we pose
    \begin{equation*}
        \mu_{a_0}(\Jcal_+):= \nu_{g_0} + \tilde z_T \quad\text{and} \quad \mu_{a_0}(\Jcal_-):= \nu_{g_0} - \tilde z_T,
    \end{equation*}
    where $\tilde z_T:=\min(z_T,1/4)$. Importantly, since the mean reward values in $\Ical$ lie in $[1/4,3/4]$, the mean reward value of $a_0$ in $\Jcal_+$ or $\Jcal_-$ still lie in $[0,1]$, hence $\Jcal_+,\Jcal_-\in\Mcal$ (in particular in the Bernoulli case, these are proper Bernoulli instances). We focus on the first $P=\floor{(16\tilde z_T)^{-2}}$ pulls of $a_0$ in these scenarios. Note that if $z_T> 1/16$ then $P=0$ hence the TV distance bound between the first $P$ pulls from $\Ical$ and $\Jcal_+$ (resp, $\Jcal_-$) in \cref{eq:bounding_TV_distance_gaussian} is vacuous. We consider the remaining case when $z_T \leq 1/16$. In that case, $\tilde z_T=z_T$ hence we still have $|\mu_{a_0}^{(0)} - \mu_{a_0}(\Jcal_+)| \leq 2\tilde z_T$ (and similarly for $\Jcal_-$). Further, in this case, $\mu_{a_0}(\Jcal_+),\mu_{a_0}(\Jcal_-)\in[1/4-1/16,3/4+1/16]$. These values are bounded away from $0$ or $1$, hence we can use similar KL bounds between the corresponding Bernoulli variable as in the Gaussian case. Altogether, in the Bernoulli case we obtain
    \begin{equation}\label{eq:bound_TV_bernoulli}
        \TV(\Ber(\mu_{a_0}(\Jcal_+))^{\otimes P},\Ber(\mu_{a_0}^{(0)})^{\otimes P}) \leq  \sqrt{\frac{P}{2}\Dkl(\mu_{a_0}(\Jcal_+)\parallel \mu_{a_0}^{(0)} )} \leq 4z_T\sqrt{P} \leq \frac{1}{4},
    \end{equation}
    and similarly for $\Jcal_-$.
    In the Gaussian case, the arguments in \cref{eq:bounding_TV_distance_gaussian} already show that the TV distance between the first $P$ pulls of $\Ical$ and $\Jcal_+$ (resp. $\Jcal_-$) is at most $1/4$.
    From there, the same arguments show that on an event of probability at least $1/4$, the run of {\sf ALG} on $\Ical,\Jcal_+,\Jcal_-$ is identical for the first $t_0$ rounds. We conclude with the same arguments that {\sf ALG} incurs at least $z_T\bar t_0/8\gtrsim R(\epsilon_T)/\log T$ regret on one instance from $\Jcal_+$ or $\Jcal_-$.
\end{proof}

\section{Proofs of minimax collaborative regret bounds}
\label{sec:minimax_bounds}

We next prove the near-optimality of \ColUCB for the minimax collaborative regret from \cref{thm:worst_case_minimax}. We recall the notation
\begin{equation*}
    \epsilon_T(\Ical):=\min\{z\in(0,1]:T(z;\Ical)\leq T\}.
\end{equation*}
Note that both functions $T(\cdot;\Ical)$ and $R(\cdot;\Ical)$ only depend on the mean values of the distributions on $\Ical$. In particular, for both Gaussian and Bernoulli models $\Mcal\in\{\Mcal_B,\Mcal_G\}$, the value
\begin{equation}\label{eq:definition_R_T_max}
    R_{T,\max}:=\sup_{\Ical\in\Mcal} R(\epsilon_T(\Ical);\Ical)
\end{equation}
coincides. As detailed in the following result, the quantity $R_{T,\max}$ characterizes the minimax collaborative regret up to logarithmic factors.

\begin{theorem}\label{thm:worst_case_minimax_bis}
    For any horizon $T\geq 2$, the minimax collaborative regret for either Gaussian of Bernoulli reward models $\Mcal\in\{\Mcal_G,\Mcal_B\}$ satisfies
    \begin{equation*}
        \frac{R_{T,\max} }{\log T} + t_{\min}' \lesssim \Rcal_T(\Gcal,\Mcal) \leq \sup_{\Ical\in\Mcal} \MaxReg_T(\ColUCB;\Ical) \lesssim (R_{T,\max}+t_{\min}')\log T.
    \end{equation*}
    where $t_{\min}':=\min(t_{\min}/\log T,T)$
\end{theorem}

\begin{proof}
    For both models considered, reward distributions are subGaussian with parameter at most $1$ and have mean rewards always lie in $[0,1]$. Hence
    \cref{thm:instance_dependent_regret} directly implies that for $\Mcal\in\{\Mcal_G,\Mcal_B\}$,
    \begin{equation*}
        \sup_{\Ical\in\Mcal}\MaxReg_T(\ColUCB;\Ical) \lesssim R_{T,\max}\log T + t_{\min}.
    \end{equation*}
    We also clearly have $\MaxReg_T(\ColUCB;\Ical)\leq T$ for all $\Ical\in\Mcal$.
    Next, we use \cref{thm:instance_lower_bound_v2} to get the desired lower bound on the minimax regret $\Rcal_T(\Gcal,\Mcal)$. We only need to make few adjustments to ensure that the considered instances have mean rewards in $[0,1]$.
    For any instance $\Ical\in\Mcal$ with mean rewards in $[0,1]$, we can consider the instance $\tilde\Ical\in\Mcal$ obtained by applying the centering function $x\in[0,1]\mapsto \frac{1}{2}\paren{x-\frac{1}{2}} +\frac{1}{2} \in [\frac{1}{4},\frac{3}{4}]$ to all mean rewards. For any $g\in\Gcal$ and $a\in\Acal_g$ we have $\Delta_{g,a}(\tilde\Ical) = \frac{1}{2} \Delta_{g,a}(\Ical)$. Hence, for $\epsilon>0$,
    $\Ccal^\star(\epsilon;\tilde\Ical) = \Ccal^\star(2\epsilon;\Ical)$ and
    \begin{align*}
M(\epsilon;\tilde \Ical) &= \max_{x\geq 0,M\geq 0}    M \quad \textrm{s.t.} \begin{cases} 
\sum_{a\in  \Acal_g} \max\paren{ \Delta_{g,a},2\epsilon} \cdot x_{g,a} \leq 2, 
 &g\in\Gcal\\
  \sum_{g\in\Gcal} x_{g,a} \geq M, &a\in \Ccal^\star(2\epsilon;\Ical),
 \end{cases}\\
&= 2 M(2\epsilon;\Ical).
    \end{align*}
    Note that since rewards of $\Ical$ are in $[0,1]$, $M(\epsilon;\Ical)$ is constant for $\epsilon\geq 1$ and hence $M(\epsilon;\tilde\Ical)$ is constant for $\epsilon\geq 1/2$. Hence, using the fact that $ M(\epsilon;\Ical)$ is non-increasing, there exist constants $0<c_0\leq  c_1$ such that for $\epsilon\in(0,1/4]$,
    \begin{align*}
          c_0T( 2\epsilon ;\Ical)&\leq T(\epsilon;\tilde\Ical) \leq c_1 T(2\epsilon;\Ical)+\frac{1}{M(1;\Ical)},\\
          c_0 R( 2\epsilon ;\Ical)&\leq R(\epsilon;\tilde\Ical) \leq c_1  R(2\epsilon;\Ical)+\frac{1}{M(1;\Ical)}.
    \end{align*}
    Next, note that since all mean values are in $[0,1]$
    \begin{equation*}
        M(1;\Ical) = \max_{x\geq 0,M\geq 0} M \;\text{s.t.}\;\begin{cases}
            \sum_{a\in\Acal_g} x_{g,a}\leq 1 &g\in\Gcal,\\
            \sum_{g\in\Gcal:a\in\Acal_g} x_{g,a}\geq M & a\in\Acal
        \end{cases}
        \quad =\frac{1}{t_0},
    \end{equation*}
    where $t_0$ is as defined in \cref{lemma:preliminary_t_min}. Further, from \cref{lemma:preliminary_t_min} we have  $t_{\min}-1\lesssim t_0\log T$.
    Without loss of generality, $c_0\leq 1$.
    Next, note that we always have $T(\epsilon;\Ical)\geq R(\epsilon;\Ical)\geq \epsilon T(\epsilon;\Ical)$ for any $\epsilon>0$. In particular, if $\epsilon_T(\tilde \Ical)\geq 1/4$, then $R(\epsilon_T(\tilde\Ical);\tilde\Ical)\asymp T \geq R(\epsilon_T(\Ical);\Ical)$. Otherwise, we can use the previous estimates which give $\epsilon_T(\tilde\Ical) \geq \min\{\epsilon: T(2\epsilon;\Ical)\leq T/c_0\}:=\delta_T=\epsilon_{T/c_0}(\Ical)/2$. Then, using the fact that $R(\cdot;\tilde\Ical)/T(\cdot;\tilde \Ical)$ is non-decreasing, we obtain
    \begin{align*}
        \frac{R(\epsilon_T(\tilde\Ical);\tilde\Ical)}{T} \geq \frac{R(\delta_T ;\tilde\Ical)}{T(\delta_T ;\tilde\Ical)} \geq c_2 \frac{R(\epsilon_{T/c_0}(\Ical) ;\Ical)}{T(\epsilon_{T/c_0}(\Ical) ;\Ical)+t_0} = c_3  \frac{R(\epsilon_{T/c_0}(\Ical);\Ical)}{T+t_0} \geq  c_3 \frac{R(\epsilon_T(\Ical);\Ical)}{T+t_0} ,
    \end{align*}
    for some universal constants $c_2,c_3>0$.
    In the last inequality we used $\epsilon_{T/c_0}(\Ical)\leq \epsilon_T(\Ical)$ since $c_0\leq 1$. In summary, provided that $t_0\lesssim T$, in all cases we have 
    \begin{equation*}
        R(\epsilon_T(\tilde\Ical);\tilde\Ical)\gtrsim R(\epsilon_T(\Ical);\Ical).
    \end{equation*}
    Then, we apply \cref{cor:instance_dependent_LB_bernoulli} to $\tilde\Ical$ which has mean arm values in $[\frac{1}{4},\frac{3}{4}]$. This gives another instance $\Jcal\in\Mcal$ obtained by modifying the value of one arm from $\tilde\Ical$ and such that
    \begin{equation*}
        \max_{I\in\{\tilde\Ical,\Jcal\}}\MaxReg_T({\sf ALG} ;I) \gtrsim \frac{R(\epsilon_T(\tilde\Ical);\tilde\Ical)}{\log T} \gtrsim \frac{R(\epsilon_T(\Ical);\Ical)}{\log T}.
    \end{equation*}
    Taking the supremum over all instances $\Ical\in\Mcal$ this precisely shows that
    \begin{equation}\label{eq:first_lower_bound}
        \Rcal_T(\Gcal,\Mcal) \gtrsim \frac{R_{T,\max}}{\log T}.
    \end{equation}
    We note that in the remaining case $t_0\gtrsim T$, the second term in the lower bound will dominate nonetheless.

    We now focus on this second term and aim to show that $\Rcal_T(\Gcal,\Mcal)\gtrsim \min(\ceil{t_0},T)$, where $t_0$ is defined as in \cref{lemma:preliminary_t_min}. On one hand, since groups have at least 2 feasible arms, we already have $\Rcal_1(\Gcal,\Mcal)\gtrsim 1$. Next, suppose that $t_0':=\floor{t_0/2}\geq 1$. Then, from the definition of $t_0$ letting $\Ical_0$ be the instance where all reward means are 0, for any algorithm {\sf ALG} there exists $a\in\Acal$ for which $\Pbb(P_a(t_0')>0)\leq 1/2$. Hence, letting $\Ical_a$ be the instance for which all reward means are 0 except that of arm $a$ which has reward $1$, we have for any $t\leq t_0$: $\MaxReg_t({\sf ALG};\Ical_a)\geq t \Pbb(P_a(t_0')=0)\geq t/2$. In summary, we have
    \begin{equation*}
        \Rcal_T(\Gcal,\Mcal) \gtrsim \min(\ceil{t_0},T).
    \end{equation*}
    We recall from \cref{lemma:preliminary_t_min} that $t_{\min}=\ceil{n_0t_0}$ where $n_0\asymp \log T$.
    Then, combining this with \cref{eq:first_lower_bound} completes the proof.
\end{proof}

In particular, \cref{thm:worst_case_minimax_bis} directly implies \cref{thm:worst_case_minimax}.
We next prove the bounds on the minimax regret in terms of the combinatorial quantities $\Hf^+(\Acal),\Hf^-(\Acal)$.

\vspace{3mm}

\begin{proof}[of \cref{thm:quantitative_minimax}]
    We prove the upper and lower bound separately.
    \paragraph{Upper bound on the maximum collaborative regret of \ColUCB.} We start with the upper bound for which we further the bounds from \cref{thm:worst_case_minimax,thm:worst_case_minimax_bis}. Precisely we give an upper bound on $R_{T,\max}$ (defined in \cref{eq:definition_R_T_max}). Fix an instance $\Ical$ with mean rewards within $[0,1]$ and any $\epsilon\in(0,1]$. Let $\Gcal':=\set{g\in\Gcal:\Delta_{g,\min}\leq \epsilon}$. For convenience, we write $\Ccal^\star:=\Ccal^\star(\epsilon;\Ical)$. For any group $g\in\Gcal$, we define $S_g:=\set{a\in\Acal_g\cap\Ccal^\star:\Delta_{g,a}\leq \epsilon}$.
    As a result,

    \begin{align*}
        M(\epsilon;\Ical) 
        &\geq \max_{x\geq 0,M\geq 0}  M \quad \textrm{s.t.} 
        \begin{cases} 
            \sum_{a\in  \Acal_g:\Delta_{g,a}> \epsilon}  x_{g,a} + \epsilon \sum_{a\in  S_g}  x_{g,a} \leq 1, &g\in\Gcal\\
            \sum_{g\in\Gcal:a\in\Acal_g} x_{g,a} \geq M, &a\in \Ccal^\star
        \end{cases} \\
        &\geq \frac{1}{2}\cdot \max_{x\geq 0,M\geq 0}  M \quad \textrm{s.t.} 
        \begin{cases} 
            \sum_{a\in  (\Acal_g\cap\Ccal^\star)\setminus S_g}  x_{g,a} \leq 1, &g\in\Gcal\\
            \sum_{a\in  S_g}  x_{g,a} \leq 1/\epsilon, &g\in\Gcal\\
            \sum_{g\in\Gcal:a\in\Acal_g} x_{g,a} \geq M, &a\in \Ccal^\star
        \end{cases} \\
        &\overset{(i)}{=} \frac{1}{2}\min_{\lambda^{(1)},\lambda^{(2)},\eta\geq 0}  \sum_{g\in\Gcal} \paren{\lambda^{(1)}_g + \frac{1}{\epsilon} \lambda^{(2)}_g} \quad \textrm{s.t.} 
        \begin{cases} 
            \lambda^{(1)}_g \geq \eta_a &g\in\Gcal,a\in (\Acal_g\cap\Ccal^\star)\setminus S_g\\
            \lambda^{(2)}_g \geq \eta_a &g\in\Gcal,a\in S_g\\
            \sum_{a\in\Ccal^\star} \eta_a=1&
        \end{cases}
    \end{align*}
    In $(i)$ we used strong duality. Let $\eta^\star$ be an optimal solution to the right-hand side problem. We decompose $\eta^\star=\sum_{i=1}^k p_i \frac{\1_{A_i}}{|A_i|}$ for $k\geq 1$, where $A_1,\ldots,A_k\subseteq \Ccal^\star$ are nested and $p\in \Delta_k=\{x\in\Rbb_+^k, \sum_{i=1}^k x_i=1\}$. Then, we obtained
    \begin{align*}
        M(\epsilon;\Ical) &\geq \frac{1}{2}\sum_{g\in\Gcal} \paren{ \max_{a\in (\Acal_g\cap\Ccal^\star)\setminus S_g} \eta_a^\star + \frac{1}{\epsilon}\max_{a\in S_g} \eta_a^\star}\\
        &=\frac{1}{2}\sum_{i=1}^k \frac{p_i}{|A_i|} \sum_{g\in\Gcal} \paren{ \1[A_i\cap (\Acal_g\setminus S_g) \neq\emptyset] + \frac{1}{\epsilon} \1[A_i\cap S_g\neq \emptyset] }\\
        &\overset{(i)}{=}\frac{1}{2} \min_{S\subseteq \Ccal^\star}\underbrace{\frac{1}{|S|} \paren{\abs{\set{g\in\Gcal:S\cap (\Acal_g\setminus S_g) \neq\emptyset}} + \frac{1}{\epsilon} \abs{\set{g\in\Gcal:S\cap S_g\neq \emptyset}}}}_{:=N(S)}
    \end{align*}
    where in $(i)$ the inequality $\leq$ is due to the fact that $\eta^\star$ is an optimal solution. Now fix $S\subseteq \Ccal^\star$. Note that by construction of $\Ccal^\star$, the set of groups $\Gcal'(S):=\set{g\in\Gcal:S\cap S_g\neq \emptyset}$ cover $S$, that is $S\subseteq \Cov(\Gcal'(S))$. Hence, by definition of $\NChelpLow(S)$, we have
    \begin{equation*}
        \abs{\Gcal'(S)} \geq |S|\cdot \NChelpLow(S) ,\quad S\subseteq \Ccal^\star.
    \end{equation*}
    Then, 
    \begin{align*}
        N(S) \geq \frac{\abs{\set{g\in\Gcal:S\cap\Acal_g\neq\emptyset}} + \paren{\frac{1}{\epsilon}-1}|\Gcal'(S)|}{|S|} \geq \help(S) + \paren{\frac{1}{\epsilon}-1} \NChelpLow(S).
    \end{align*}
    As a result, we obtained for any instance $\Ical$ and $\epsilon\in(0,1]$,
    \begin{equation}\label{eq:main_lower_bound_M}
        M(\epsilon;\Ical) \geq \frac{1}{2} \min_{S\subseteq \Acal} \set{ \help(S) + \paren{\frac{1}{\epsilon}-1} \NChelpLow(S) }:=\phi(\epsilon).
    \end{equation}

    Define by $\Dcal(\Ical)$ the distribution on $[\epsilon_T(\Ical),1]$ with density $1/(TM(z;\Ical)z^4)$ for $z\in [\epsilon_T(\Ical),1]$. Note that this is a well defined distribution since its integral is $T(\epsilon_T(\Ical);\Ical)/T=1$ from \cref{lemma:continuity_R_T}. By definition of $T(\cdot;\Ical)$ and $R(\cdot;\Ical)$, we can write
    \begin{equation*}
        \frac{R(\epsilon_T(\Ical);\Ical)}{T}=\Ebb_{z\sim\Dcal(\Ical)}[z].
    \end{equation*}

    Next, we define $\delta_T\in(0,1]$ such that $\int_\epsilon^1 \frac{dz}{\phi(z)z^4}= T$. We define the distribution $\widetilde\Dcal$ as the distribution on $[\delta_T,1]$ with density $1/(T\phi(z)z^4)$ for $z\in [\delta_T,1]$. From \cref{lemma:continuity_R_T}, we can check that $\widetilde\Dcal(\Ical)$ stochastically dominates $\Dcal(\Ical)$. As a result,
    \begin{equation*}
        R(\epsilon_T(\Ical);\Ical) \leq T\cdot \Ebb_{z\sim\widetilde\Dcal}[z] =  \int_{\delta_T}^1 \frac{dz}{\phi(z)z^3}:=\widetilde R_T.
    \end{equation*}

    In other terms, it suffices to bound the desired regret integral as if we replaced $M(\epsilon;\Ical)$ with $\phi(\epsilon)$. Observing that $\epsilon\phi(\epsilon)$ is non-decreasing for $\epsilon\in(0,1]$ while $\phi(\epsilon)$ is non-increasing in $\epsilon\in(0,1]$, we can write
    \begin{equation*}\label{eq:bound_tilde_R}
        \widetilde R_T \leq \frac{1}{\phi(\delta_T) \delta_T}\int_{\delta_T}^1 \frac{dz}{z^2} = \frac{1-\delta_T}{\phi(\delta_T) \delta_T^2}.
    \end{equation*}
    Importantly, note that since $\phi$ does not depend on $\Ical$, $\widetilde R_T$ is also independent of $\Ical$.
    We now lower bound $T$ using
    \begin{align*}
        T= \int_\epsilon^1 \frac{dz}{\phi(z)z^4} \geq \frac{1}{ \phi(\delta_T)} \int_{\delta_T}^1 \frac{dz}{z^4} \geq \frac{1-\delta_T}{ 3\phi(\delta_T) \delta_T^3}.
    \end{align*}
    Altogether, we obtained
    \begin{equation}\label{eq:first_hypothesis}
        R_{T,\max}=\sup_{\Ical\in\Mcal} R(\epsilon_T(\Ical);\Ical)\leq \widetilde R_T \lesssim \frac{T^{2/3}}{\phi(\delta_T)^{1/3}}.
    \end{equation}
The following step is to bound $\phi(\delta_T) \gtrsim \Hf_T^-(\Acal)$. We distinguish between two cases.

\paragraph{Main case.} We suppose for now that $\delta_T \leq \frac{1}{2}$.
 Let $S\subseteq\Acal$ such that
    \begin{equation}\label{eq:loweR_bound_phi_delta_T}
        \phi(\delta_T) = \frac{1}{2} \paren{ \help(S) + \paren{\frac{1}{\delta_T}-1} \NChelpLow(S) } \gtrsim \help(S) +  \frac{\NChelpLow(S)}{\delta_T},
    \end{equation}
    where in the last inequality we used $\delta_T\leq \frac{1}{2}$.
    Next, we upper bound $T$ as follows:
    \begin{align*}
        T=\int_{\delta_T}^1\frac{dz}{\phi(z)z^4} \leq \frac{1}{\phi(\delta_T)\delta_T}\int_{\delta_T}^1\frac{dz}{z^3} = \frac{1}{2\phi(\delta_T) \delta_T^3}  \overset{(i)}{\lesssim} \frac{1}{\NChelpLow(S_T) \delta_T^2}.
    \end{align*}
    In $(i)$ we used \cref{eq:loweR_bound_phi_delta_T}. Reordering this gives
    \begin{equation}\label{eq:third_hypothesis}
        \sqrt{\NChelpLow(S) T } \lesssim \frac{1}{\delta_T}.
    \end{equation}
    Furthering the bound, we obtain
    \begin{equation*}
        \phi(\delta_T) \gtrsim \help(S)+ \frac{\NChelpLow(S)}{\delta_T} \overset{(i)}{\gtrsim} \help(S) + \NChelpLow(S)^{3/2} \sqrt T \geq \Hf_T^-(\Acal).
    \end{equation*}
    In $(i)$ we used \cref{eq:third_hypothesis}.
    
    \paragraph{Edge case.} We next turn to the case when $\delta_T>1/2$. Therefore this means that
    \begin{equation*}
        T =\int_{\delta_T}^1\frac{dz}{\phi(z)z^4} \leq \int_{1/2}^1 \frac{dz}{\phi(z)z^4} \leq \frac{3}{\phi(1/2)},
    \end{equation*}
    where we again used that $\epsilon\phi(\epsilon)$ is non-decreasing in $\epsilon\in(0,1]$. Going back to the definition of $\phi(1/2)$, this shows that there exists $S\subseteq\Acal$ such that $T\help(S)+T\NChelpLow(S)\lesssim 1$. As a result,
    \begin{equation*}
        T\Hf_T^-(\Acal) \leq T\help(S) + (T\NChelpLow(S))^{3/2} \lesssim 1.
    \end{equation*}
    As a result, we directly have $\tilde R_T\leq T\lesssim\frac{T^{2/3}}{\Hf_T^-(\Acal)^{1/3}}$. 
    
    In summary, in all cases, altogether with \cref{eq:first_hypothesis} we obtained
    \begin{equation*}
        R_{T,\max}\lesssim \frac{T^{2/3}}{\Hf_T^-(\Acal)^{1/3}}.
    \end{equation*}
    To achieve the desired upper bound on the minimax regret $\Rcal_T(\Gcal,\Mcal)$ using \cref{thm:worst_case_minimax_bis}, it only remains to bound $t_{\min}$. From \cref{lemma:preliminary_t_min}, we have $t_{min}'\lesssim\min(\ceil{t_0},T)$ and there exists $\Gcal'\subseteq\Gcal$ with $t_0 = |S|/|\Gcal'|$, where $S:=\Cov(\Gcal')\setminus \Cov(\Gcal\setminus\Gcal')$. Now note that $\{g\in\Gcal:\Acal_g\cap S\neq \emptyset \}\subseteq \Gcal'$. Hence,
    \begin{equation*}
        \Hf_T^-(\Acal) \leq \help(S) + \NChelpLow(S)^{3/2} \sqrt T\leq \frac{|\Gcal'|}{|S|} + \paren{\frac{|\Gcal'|}{|S|}}^{3/2} \sqrt T\leq \frac{1}{t_0} + \frac{\sqrt T}{t_0^{3/2}} \leq \frac{T^2}{\min(t_0,T)^3}.
    \end{equation*}
    In particular, $\frac{T^{2/3}}{\Hf_T^-(\Acal)^{1/3}}\gtrsim \min(t_0,T)$. We conclude the proof of the upper bound on $\Rcal_T(\Gcal,\Mcal)$ using \cref{thm:worst_case_minimax_bis}.

    \paragraph{Lower bound on the minimax collaborative regret.} As for the upper bound, it suffices to obtain a lower bound on $R_{T,\max}+t_{\min}'$ and we can then use the lower bound from \cref{thm:worst_case_minimax_bis}. This would however result in an extra factor $1/\log T$, which is due to the same factor in the lower bound of \cref{thm:instance_lower_bound_v2}. To remove this factor, we slightly adjust the proof of this lower bound. We fix either the Gaussian or Bernoulli model $\Mcal\in\{\Mcal_G,\Mcal_B\}$.

    We fix $S\subseteq \Acal$ such that $\Hf_T(\Acal) = \combined_T(S)$. By definition of $\NChelp(S)$ there exists a subset of groups $\Gcal_S\subseteq \Gcal$ that covers $S$, that is $S\subseteq \Cov(\Gcal_S)$ and letting $S_c:= \bigcup_{g\in\Gcal_S}\Acal_g$ we have
    \begin{equation*}
        \NChelpHigh(S) = \frac{\abs{\set{g\in\Gcal:\Acal_g\cap S\neq\emptyset,\Acal_g\subseteq S_c}}}{|S|} .
    \end{equation*}
    We focus on a specific set of instances for the rewards of each arm. For any $a_0\in S\cup\{\emptyset\}$ we define the instance $\Ical\in\Mcal$ in which
    \begin{equation*}
        (\Ical_{a_0}(\epsilon))\qquad \qquad \mu_a=\begin{cases}
            1 &a\in\Acal\setminus S_c,\\
            0 &a\in S_c\setminus S,\\
            1/2 &a\in S,
        \end{cases}
    \end{equation*}
    for $\epsilon\in(0,1/2]$. The main bottleneck to minimize the collaborative regret in $\Ical$ is to explore the arms in $S$. More precisely:
    \begin{itemize}
        \item Groups in $\Gcal_1:=\set{g\in\Gcal: \Acal_g\cap S\neq\emptyset, \Acal_g\nsubseteq S_c}$ can participate in the shared exploration of some arms in $S$ but incur a regret $1/2$ for each such exploration.
        \item Groups in $\Gcal_2:=\set{g\in\Gcal: \Acal_g\cap S\neq\emptyset, \Acal_g\subseteq S_c}$ can participate in the shared exploration of some arms in $S$ for free.
    \end{itemize}
    The number of these groups is quantified by $\help(S)$ and $\NChelpHigh(S)$ respectively.
    We fix a strategy {\sf ALG} for the groups. Then, we can bound the number of pulls of arms in $S$ in the instance $\Ical$ as follows:
    \begin{align*}
        \sum_{t=1}^T\sum_{g\in\Gcal} \1[a_g(t)\in S] &\leq  \sum_{g\in\Gcal_1} \sum_{t=1}^T \1[a_g(t)\in S] + T|\Gcal_2| \\
        &\leq 2\sum_{g\in\Gcal_1} \Reg_{g,T}({\sf ALG};\Ical) + T|\Gcal_2| \\
        &\leq |S| \paren{ 2\help(S)\cdot \max_{g\in\Gcal} \Reg_{g,T}({\sf ALG};\Ical) + T\NChelpHigh(S) },
    \end{align*}
    where in the last inequality we used the definition of $\help(S)$ and $\NChelpHigh(S)$. We next denote by $P_a(T):= \abs{\set{(t,g): t\in[T],g\in\Gcal,a_g(t)=a}}$ the total number of pulls of arm $a\in\Acal$. Taking the expectation gives
    \begin{equation*}
        \frac{1}{S}\sum_{a\in S} \Ebb_{\Ical}\sqb{P_a(T)} \leq 2\help(S)\cdot \MaxReg_T({\sf ALG};\Ical) + T\NChelpHigh(S).
    \end{equation*}
    Therefore, we can fix $a_0\in S$ such that
    \begin{equation}\label{eq:expectation_pull_bound}
        \Ebb_{\Ical_\emptyset (\epsilon)}\sqb{P_{a_0}(T)} \leq 2\help(S)\cdot \MaxReg_T({\sf ALG};\Ical) + T\NChelpHigh(S).
    \end{equation}
    From there, we can use similar arguments as in the instance-dependent lower bound from \cref{thm:instance_lower_bound_v2}. For $\epsilon\in[0,1/2]$, we consider the instances $\Jcal_+(\epsilon)$ and $\Jcal_-(\epsilon)$ obtained by keeping all arm values identical as in $\Ical$ except for $a_0$ which has $\mu_{a_0}(\Jcal_+):=1/2+\epsilon$ and $\mu_{a_0}(\Jcal_+):=1/2-\epsilon$. We consider the first $P$ pulls of arm $a_0$ and bound the TV distance of these pulls between $\Ical$ and $\Jcal_+(\epsilon)$ (resp. $\Jcal_-(\epsilon)$). For both Gaussian or Bernoulli models, this TV distance is bounded by $\epsilon\sqrt {2P}$,
    e.g., see \cref{eq:bounding_TV_distance_gaussian,eq:bound_TV_bernoulli}. As a result, there exists an event $\Ecal(P,\epsilon)$ of probability at least $1-2\epsilon\sqrt{2P}$ on which the first $P$ pulls of $\Ical,\Jcal_+(\epsilon),\Jcal_-(\epsilon)$ are identical. In summary, on $\Fcal:=\Ecal(P,\epsilon)\cap\{P_{a_0}(T)\leq P\}$, the run of {\sf ALG} on $\Ical,\Jcal_+(\epsilon),\Jcal_-(\epsilon)$ is identical. Next, we fix the parameters
    \begin{equation*}
        P:= 8\help(S)\cdot \MaxReg_T({\sf ALG};\Ical) + 4T\NChelpHigh(S) \quad \text{and} \quad \epsilon = \min\paren{\frac{1}{8\sqrt {2P}}, \frac{1}{2}}.
    \end{equation*}

    Then, using \cref{eq:expectation_pull_bound} and Markov's inequality, and we have
    \begin{align*}
        \Pbb(\Hcal) \geq 1-\Pbb_{\Ical}(P_{a_0}(T)> P) - 2\epsilon\sqrt {2P} \geq \frac{1}{2}.
    \end{align*}
    Since $\Gcal_S$ covers $S$ there exists $g_0\in\Gcal_S$ such that $a_0\in \Acal_{g_0}$. 
    Under $\Jcal_+(\epsilon)$ (resp. $\Jcal_-(\epsilon)$), group $g_0$ incurs an $\epsilon$ regret for each iteration that it does not pull arm $a_0$ (resp. pulls arm $a_0$):
    \begin{equation*}
        \Reg_{g_0,T}({\sf ALG};\Ical_{a_0}(\epsilon)) \geq \epsilon T - \epsilon \sum_{t=1}^T\1[a_{g_0}(t)=a_0]  \text{ and } \Reg_{g_0,T}({\sf ALG};\Jcal_{a_0}(\epsilon)) \geq \epsilon \sum_{t=1}^T\1[a_{g_0}(t)=a_0].
    \end{equation*}
    Since the run of {\sf ALG} on $\Jcal_+(\epsilon)$ and $\Jcal_-(\epsilon)$ coincide under $\Fcal$, we can sum the previous equations under $\Hcal$ to obtain
    \begin{equation*}
        \Ebb_{\Jcal_+(\epsilon)}\sqb{\Reg_{g_0,T}({\sf ALG};\Jcal_+(\epsilon))} + \Ebb_{\Jcal_-(\epsilon)}\sqb{\Reg_{g_0,T}({\sf ALG};\Jcal_-(\epsilon))} \geq \Pbb(\Fcal)\cdot \epsilon T\geq \frac{\epsilon T}{2}.
    \end{equation*}
    As a result, under at least one instance from $\Jcal_+(\epsilon)$ and $\Jcal_-(\epsilon)$, {\sf ALG} incurs a collaborative regret of at least $\epsilon T/8$. In summary, we proved
    \begin{align*}
        \sup_{\Ical\in\Mcal} \MaxReg_T({\sf ALG};\Ical) \geq \max\set{\MaxReg_T({\sf ALG};\Ical), \frac{\epsilon T}{8}}.
    \end{align*}
    If $1/(8\sqrt{2P})\geq 1/2$ then $\epsilon=1/2$ and this shows that $\sup_{\Ical\in\Mcal} \MaxReg_T({\sf ALG};\Ical) \asymp T$. Otherwise, 
    \begin{equation*}
        \sup_{\Ical\in\Mcal} \MaxReg_T({\sf ALG};\Ical)\geq \inf_{z\geq 0} \max\set{z,\frac{T}{2^8\sqrt{z \help(S) + T\NChelpHigh(S)} }}.
    \end{equation*}
    Therefore, we obtain
    \begin{equation*}
        \sup_{\Ical\in\Mcal} \MaxReg_T({\sf ALG};\Ical) \gtrsim \begin{cases}
            \frac{T^{2/3}}{ \help(S)^{1/3}} &\text{if } T\help(S)\geq (T \NChelpHigh(S))^{3/2}\\
            \sqrt{\frac{T}{\NChelpHigh(S)}} &\text{otherwise}
        \end{cases} \geq \frac{T^{2/3}}{\combined_T^+(S)^{1/3}}=\frac{T^{2/3}}{\Hf_T(\Acal)^{1/3}}.
    \end{equation*}
    This lower bound holds for any strategy {\sf ALG} for groups, which yields the desired lower bound on the minimax regret $\Rcal_T(\Gcal)$.
    This ends the proof.
\end{proof}

\section{Helper results}

We start by proving the simple continuity result on the functionals $T$ and $R$.

\begin{lemma}\label{lemma:continuity_R_T}
    For any instance $\Ical$ and $\sigma>0$, $T(\cdot;\Ical,\sigma)$ and $R(\cdot;\Ical,\sigma)$ are continuous and non-increasing on $(0,1]$, and $\lim_{\epsilon\to 0^+} T(\epsilon;\Ical,\sigma) = \lim_{\epsilon\to 0^+} R(\epsilon;\Ical,\sigma) =+\infty$. In particular, $\epsilon=\min\{z\in(0,1]:T(z;\Ical,\sigma)\leq T\}$ satisfies $T(\epsilon;\Ical,\sigma)=T$.
\end{lemma}

\begin{proof}
    The function $M(\cdot;\Ical,\sigma)$ is non-increasing in $\epsilon$ and finite for all $\epsilon>0$. Hence, $T(\cdot;\Ical,\sigma)$ and $R(\cdot;\Ical,\sigma)$ are continuous and non-increasing. 
    Next, we always have $M(\epsilon;\Ical)\leq |\Gcal|/\epsilon$. Hence $\lim_{\epsilon\to 0^+} T(\epsilon;\Ical,\sigma) = \lim_{\epsilon\to 0^+} R(\epsilon;\Ical,\sigma) =+\infty$.
\end{proof}

\subsection{Burn-in period for \ColUCB}

Next, we give some more details on the burn-in period for \ColUCB during the first $t_{\min}$ rounds. The following lemma further chracterizes $t_{\min}$ and we describe within the proof a simple greedy procedure to ensure that within $t_{\min}$ rounds each arm is pulled at least $16C\log T$ rounds.

\begin{lemma}\label{lemma:preliminary_t_min}
    Consider the following minimization problem for any $n\geq 1$:
    \begin{align*}
t_0&:= \min_{x\geq 0}  t \;\;
\textrm{s.t.} \;\begin{cases}
    \sum_{a\in \Acal_g} x_{g,a} \leq t, 
 &g\in\Gcal\\
 \sum_{g\in\Gcal:a\in\Acal_g} x_{g,a} \geq 1, &a\in \Acal
\end{cases} \quad = \max_{\emptyset\subsetneq\Gcal'\subseteq \Gcal} \frac{\abs{\Cov(\Gcal')\setminus \Cov(\Gcal\setminus\Gcal')}}{|\Gcal'|}.
    \end{align*}
    Then, $t_{\min}=\ceil{n_0\cdot t_0}$, and a schedule to that all arms at least $16C\log T$ times by the end of iteration $t_{\min}$ can be easily computed as the solution to a network flow problem.
\end{lemma}

\begin{proof}
    Note that $n\cdot t_0$ corresponds to the LP relaxation solution to finding the minimum time to ensure that each arm is pulled at least $n$ times. In particular, we directly have $t_{\min} \geq n_0 t_0$, that is $t_{\min}\geq \ceil{n_0t_0}$. We now prove the equivalent characterization of $t_0$ which will be useful within the proofs. This is a variant of Hall's mariage theorem. By strong duality, for any $n\geq 1$
    \begin{align*}
        t_0&:= \min_{x\geq 0}  t \;\;
            \textrm{s.t.} \;\begin{cases}
            \sum_{a\in \Acal_g} x_{g,a} \leq t, &g\in\Gcal\\
            \sum_{g\in\Gcal:a\in\Acal_g} x_{g,a} \geq 1, &a\in \Acal
            \end{cases}\\
            &=\max_{\eta,\lambda\geq 0} \sum_{a\in\Acal}\eta_a \;\;\textrm{s.t.}\;\begin{cases}
                \lambda_g\geq \eta_a & g\in\Gcal,a\in\Acal_g\\
                \sum_{g\in\Gcal}\lambda_g=1 .
            \end{cases}\\
            &=\max_{\lambda\in\Delta_\Gcal} \sum_{a\in\Acal} \min_{g:a\in\Acal_g} \lambda_g,
    \end{align*}
    where $\Delta_\Gcal:=\{z\in\Rbb_+^\Gcal:\sum_{g\in\Gcal}z_g=1\}$ is the simplex on groups $\Gcal$. We decompose an optimal solution $\lambda^\star\in\Delta_\Gcal$ into $\lambda^\star = \sum_{k\in[K]} \frac{p_k}{|\Gcal_k|}\1_{\Gcal_k}$ where $K\geq 1$, $\emptyset\subsetneq\Gcal_1\subsetneq\ldots \subsetneq\Gcal_K\subseteq \Gcal$ are nested, $p\in\Delta_K$, and $p_1,\ldots,p_K>0$. Then,
    \begin{align*}
        t_0 = \sum_{k\in[K]} p_k\cdot \frac{\abs{\Cov(\Gcal_k)\setminus \Cov(\Gcal\setminus \Gcal_k)}}{|\Gcal_k|} = \max_{\emptyset\subsetneq\Gcal'\subseteq \Gcal} \frac{\abs{\Cov(\Gcal')\setminus \Cov(\Gcal\setminus\Gcal')}}{|\Gcal'|}.
    \end{align*}
    In the last equality, the direction $\leq$ is due to the fact that $\lambda^\star$ is an optimal solution.

    Next, note that the problem of finding a schedule of pulls such that within $t':=\ceil{n_0 t_0}$ iterations each arm is pulled at least $n_0$ times can be formulated as a flow problem with integer capacity constraints (essentially the matching problem): each group $g\in\Gcal$ has inflow at most $t'$, each arm $a\in\Acal$ has outlow at most $n$, and for any $g\in\Gcal$ and $a\in\Acal_g$ we include an uncapacitated edge $(g,a)$. By construction, the maximum flow $n_0|\Acal|$ is feasible, hence can construct an integer flow such that each arm has outflow exactly $n_0$. For each group $g\in\Gcal$ this gives an allocation of at most $t'$ pulls to arms in $\Acal_g$. Any schedule in which each group $g$ pulls their corresponding arms (in corresponding number) therefore pulls each arm at least $n_0$ times after at $t'$ rounds.
\end{proof}

\subsection{Sufficient condition for improved collaborative regret}

\begin{lemma}\label{lemma:sufficient_condition_improve}
    Fix a horizon $T\geq 1$. If for $\alpha\geq 1$,
    \begin{equation*}
        \min_{\emptyset\subsetneq S\subseteq\Acal} \help(S) \geq \alpha \frac{\sqrt T}{\max_{g\in\Gcal}|\Acal_g|^{3/2}},
    \end{equation*}
    then the worst-case collaborative regret for Gaussian bandits satisfies
    \begin{equation*}
        \sup_{\Ical}\MaxReg_T(\ColUCB;\Ical) \leq \frac{\log T}{\alpha^{3/2}}\sqrt{\min_{g\in\Gcal} |\Acal_g| \cdot T} .
    \end{equation*}
\end{lemma}
\begin{proof}
    This is immediate from $\Hf_T^-(\Acal) \geq \min_{\emptyset\subsetneq S\subseteq \Acal} \help(S)$ and \cref{thm:quantitative_minimax}.
\end{proof}

\subsection{Simplified statements under \cref{condition:nice_instance}}
\label{subsec:simplified_statement}

As discussed in the main body of the paper, for instances $\Ical$ satisfying the technical condition \cref{condition:nice_instance}, the integrals defining the functions $T(\cdot;\Ical)$ and $R(\cdot;\Ical)$ are significantly simplified.

\vspace{3mm}

\begin{proof}[of \cref{cor:simplified_instance_dependent_regret}]
    Fix an instance $\Ical$. We recall that $M(\epsilon;\Ical)$ is non-increasing in $\Ical$. Then, for any $\epsilon\in(0,1/2]$, we have
    \begin{equation*}
        T(\epsilon;\Ical) \geq \frac{1}{ M(\epsilon;\Ical)} \int_\epsilon^1 \frac{dz}{z^4} \geq  \frac{1}{3M(\epsilon;\Ical)\epsilon^3}.
    \end{equation*}
    On the other hand, if $\Ical$ satisfies \cref{condition:nice_instance},
    \begin{equation*}
        T(\epsilon;\Ical) \leq \frac{1}{ M(\epsilon;\Ical)} \int_\epsilon^1 \frac{C_1(z/\epsilon)^{2-\alpha} dz}{z^4} \leq \frac{c_0C_1}{ M(\epsilon;\Ical)\epsilon^3},
    \end{equation*}
    for some universal constant $c_0\geq 1$.
    In summary, up to the constant factor $C_1$, we have $T(\epsilon;\Ical)\approx \frac{1}{M(\epsilon;\Ical)\epsilon^3}$. We can similarly give bounds on $R(\epsilon;\Ical)$ for $\epsilon\in(0,1/2]$:
    \begin{equation*}
        \frac{1}{M(\epsilon;\Ical)\epsilon^2} \lesssim R(\epsilon;\Ical) \lesssim \frac{C_1}{\alpha M(\epsilon;\Ical)\epsilon^2} \lesssim \frac{C_1}{\alpha} \cdot \epsilon  T(\epsilon;\Ical).
    \end{equation*}
    Combining these estimates then shows that if $\epsilon^\star:=\min\{z\in(0,1]: M(z;\Ical)\cdot z^3 T\geq 1\}\cup\{1\}$ satisfies $\epsilon^\star\leq 1/2$, then
    \begin{equation*}
        \epsilon_{T/(c_0C_1)}(\Ical) =\min\{z\in(0,1]: T(z;\Ical)\leq T/(c_0C_1)\} \leq \epsilon^\star,
    \end{equation*}
    for some universal constant $c_0>0$.
    As a result, since $R(\cdot;\Ical)/T(\cdot;\Ical)$ is non-decreasing, we have
    \begin{equation}\label{eq:simplified_regret_form}
        \frac{R(\epsilon_T(\Ical);\Ical) }{T} \leq \frac{R(\epsilon_{T/(c_0C_1)} (\Ical);\Ical) }{T/(c_0C_1)} \lesssim \frac{C_1^2}{\alpha} \epsilon_{T/(c_0C_1)} (\Ical) \leq \frac{C_1^2}{\alpha} \epsilon^\star.
    \end{equation}
    In the other case when $\epsilon^\star\geq 1/2$, we immediately have $R(\epsilon_T(\Ical);\Ical)\leq T\lesssim \epsilon^\star T$. Hence, we proved that in all cases, \cref{eq:simplified_regret_form} holds. We conclude by applying \cref{thm:instance_lower_bound_v2}.
\end{proof}

\subsection{Concentration inequalities}

Throughout the paper, we use Freedman's inequality \cite{freedman1975tail} which gives tail probability bounds for martingales. The following statement is for instance taken from \cite[Theorem 1]{beygelzimer2011contextual} or \cite[Lemma 9]{agarwal2014taming}.
    
    \begin{theorem}[Freedman's inequality]\label{thm:freedman_inequality}
        Let $(Z_t)_{t\in T}$ be a real-valued martingale difference sequence adapted to filtration $(\Fcal_t)_t$. If $|Z_t|\leq R$ almost surely, then for any $\eta\in(0,1/R)$ it holds that with probability at least $1-\delta$,
        \begin{equation*}
            \sum_{t=1}^T Z_t \leq \eta \sum_{t=1}^T \Ebb[Z_t^2\mid\Fcal_{t-1}] + \frac{\log 1/\delta}{\eta}.
        \end{equation*}
    \end{theorem}

\end{document}